\documentclass{article}

\usepackage[svgnames]{xcolor}         
\PassOptionsToPackage{numbers, compress, sort}{natbib}
\PassOptionsToPackage{colorlinks=true, linkcolor=blue, citecolor=Green, pdfborder={0 0 0}}{hyperref}
\usepackage[final]{neurips_2025}
\usepackage[utf8]{inputenc} 
\usepackage[T1]{fontenc}    
\usepackage{hyperref}       
\usepackage{url}            
\usepackage{booktabs}       
\usepackage{amsfonts}       
\usepackage{nicefrac}       
\usepackage{microtype}      

\usepackage[utf8]{inputenc}
\usepackage{amsmath, amssymb, amsthm}
\usepackage{thmtools}
\usepackage{booktabs}
\usepackage{graphicx}
\usepackage{subfigure}
\usepackage[ruled,vlined]{algorithm2e}
\usepackage{caption}

\renewcommand{\sectionautorefname}{Sec.}
\renewcommand{\subsectionautorefname}{Sec.}

\newcommand*{\argmax}{\operatornamewithlimits{argmax}}
\newcommand*{\argmin}{\operatornamewithlimits{argmin}}

\newcommand*{\field}[1]{\mathbb{\MakeUppercase{#1}}}		
\newcommand*{\set}[1]{{\mathcal{\MakeUppercase{#1}}}}			
\newcommand*{\norm}[1]{\lVert #1 \rVert}	
\newcommand*{\flexnorm}[1]{\left\lVert #1 \right\rVert}
\newcommand*{\inner}[1]{\langle #1 \rangle}	
\newcommand*{\card}[1]{\lvert #1 \rvert}

\newcommand*{\R}{\field{R}} 
\newcommand*{\N}{\field{N}} 
\newcommand*{\C}{\field{C}} 
\newcommand*{\Z}{\field{Z}} 

\newcommand*{\imunit}{\iota}    

\renewcommand{\vec}[1]{{\boldsymbol{\mathbf{#1}}}}
\newcommand*{\mat}[1]{\vec{\MakeUppercase{#1}}}
\newcommand*{\eye}{\mat{I}}							
\newcommand*{\transpose}{\mathsf{T}}
\newcommand*{\tr}{{\operatorname{Tr}}}						
\newcommand*{\vecmean}[1]{\vec{\widehat{#1}}}
\newcommand*{\eigval}{\lambda}

\newcommand*{\idop}{\operator{I}}

\newcommand*{\cspace}{\set{C}}
\newcommand*{\lpspace}[1]{\set{L}^{#1}}
\newcommand*{\lopspace}{\set{L}}
\newcommand*{\functional}[1]{{\MakeUppercase{#1}}}

\newcommand*{\operator}[1]{{{\MakeUppercase{#1}}}}

\newcommand*{\dataset}{\set{D}} 				
\newcommand*{\observation}{y} 					
\newcommand*{\observations}{\vec{\observation}}
\newcommand*{\obsnoise}{\epsilon}

\newcommand*{\gp}{\mathcal{GP}}
\newcommand*{\gpmean}{\mu}

\newcommand*{\gpfunction}{g}
\newcommand*{\mig}{\gamma}

\newcommand*{\kernel}{k}
\newcommand*{\vkernel}{\vec{\kernel}}
\newcommand*{\Kernel}{\mat{\kernel}}

\newcommand*{\slocation}{x} 
\newcommand*{\location}{{\slocation}} 
\newcommand*{\domain}{\set{X}} 

\newcommand*{\dimension}{d}

\newcommand*{\measure}[1]{\MakeUppercase{#1}}
\newcommand*{\expectation}{\mathbb{E}}
\newcommand*{\variance}{\mathbb{V}}
\newcommand*{\covariance}{\operatorname{Cov}}
\newcommand*{\pmeasure}{\measure{P}}

\newcommand*{\prob}[1]{\mathbb{P}\left[ #1 \right]}

\newcommand*{\normal}{\mathcal{N}}					

\newcommand*{\diff}{{\mathop{}\operatorname{d}}}

\newcommand*{\lpmeasure}{\nu}

\newcommand*{\Hspace}{\set{H}} 						

\newcommand*{\feature}{\phi}
\newcommand*{\features}{\mat{\Phi}}
\newcommand*{\ffeature}{\vec{\feature}}

\newcommand*{\regFactor}{\lambda}

\newcommand*{\Sspace}{\set{S}} 				
\newcommand*{\regret}{r}					
\newcommand*{\Regret}{\functional{R}}					
\newcommand*{\bcr}{\Regret}

\newcommand*{\objective}{f}

\newcommand*{\vnoise}{\xi}
\newcommand*{\vnCov}{\operator{\Sigma}}

\newcommand*{\obsspace}{\set{Y}}

\newcommand*{\obsdim}{m}

\newcommand*{\outfun}{u}

\newcommand*{\infun}{a}

\newcommand*{\opkernel}{\operator{K}}

\newcommand*{\model}{h}
\newcommand*{\parameter}{\theta}
\newcommand*{\parameters}{\vec{\parameter}}
\newcommand*{\paramdim}{M}
\newcommand*{\paramspace}{\set{W}} 
\newcommand*{\paramcov}{\mat{\Sigma}}
\newcommand*{\weight}{w}    
\newcommand*{\weights}{\vec{\weight}}    
\newcommand*{\Weights}{\mat{\weight}}

\newcommand{\nnop}{\operator{G}}             
\newcommand*{\nnbias}{b}                     
\newcommand*{\nnkernel}{\mat{R}}              
\newcommand*{\infspace}{\set{A}}
\newcommand*{\outfspace}{\set{U}}
\newcommand*{\layer}{l}                     
\newcommand*{\nlayers}{L}
\newcommand*{\actfun}{\alpha}
\newcommand*{\indim}{{\dimension_{\infun}}}
\newcommand*{\outdim}{{\dimension_{\outfun}}}
\newcommand*{\nnproj}{\operator{\Pi}}
\newcommand*{\nnkop}{\operator{A}_\nnkernel}           
\newcommand*{\Fourier}{\operator{F}}
\newcommand*{\freq}{s}
\newcommand*{\nninput}{\vec{v}}
\newcommand*{\nndomain}{\set{V}}
\newcommand*{\outdomain}{\set{Z}}
\newcommand*{\outlocation}{z}
\newcommand*{\obsop}{\operator{H}}
\newcommand*{\loss}{\ell}

\newcommand{\ntk}{\mathtt{NTK}}
\newcommand{\nngp}{\mathtt{NNGP}}
\newcommand*{\radius}{R}
\newcommand*{\ball}{\set{B}}

\newcommand*{\iterIdx}{t}

\newcommand*{\nIterations}{T}
\newcommand*{\nObs}{N}

\newcommand*{\bigo}{\set{O}}        

\newcommand*{\anyconstant}{c}

\def\subsectionautorefname{Section}

\declaretheorem[style=plain]{theorem, definition, proposition, lemma, corollary, remark, assumption}

\title{Thompson Sampling in Function Spaces\\via Neural Operators}

\author{%
    Rafael Oliveira\thanks{Corresponding author: \texttt{rafael.dossantosdeoliveira@data61.csiro.au}}\\
    CSIRO's Data61\\
    Sydney, Australia\\
    \And
    Xuesong Wang\\
    CSIRO's Data61\\
    Sydney, Australia\\
    \And
    Kian Ming A. Chai\\
    DSO National Laboratories\\
    Singapore\\
    \And
    Edwin V. Bonilla\\
    CSIRO's Data61\\
    Sydney, Australia\\
}

\begin{document}
\maketitle

\begin{abstract}
We propose an extension of Thompson sampling to optimization problems over function spaces where the objective is a known functional of an unknown operator's output. We assume that queries to the operator (such as running a high-fidelity simulator or physical experiment) are costly, while functional evaluations on the operator's output are inexpensive. Our algorithm employs a sample-then-optimize approach using neural operator surrogates. This strategy avoids explicit uncertainty quantification by treating trained neural operators as approximate samples from a Gaussian process (GP) posterior. We derive regret bounds and theoretical results connecting neural operators with GPs in infinite-dimensional settings. Experiments benchmark our method against other Bayesian optimization baselines on functional optimization tasks involving partial differential equations of physical systems, demonstrating better sample efficiency and significant performance gains.
\end{abstract}

\section{Introduction}
Neural operators have established themselves as versatile models capable of learning complex, nonlinear mappings between function spaces \citep{Kovachki2023}. They have demonstrated success across diverse fields, including climate science \citep{Kurth2023fourcastnet}, materials engineering \citep{Oommen2024}, and computational fluid dynamics \citep{Li2023gino}. Although their applications in supervised learning and physical system emulation are well-studied, their potential for online learning and optimization within infinite-dimensional function spaces remains relatively untapped.

In many scientific contexts, learning operators that map between function spaces naturally arises, such as the task of approximating solution operators for a partial differential equation (PDE) \citep{Kovachki2023}. However, adaptive methods that efficiently query these operators to optimize functional objectives of their outputs (particularly in an active learning setting) are still underdeveloped. 
For example, when designing porous structures, one is often interested in optimizing how liquids flow through the structure using, e.g., Darcy flow PDEs \citep{wiker2007topology}, and, in the sciences, inverse problems can be solved by optimization to infer initial conditions or parameters of a physical process from observations \citep{Penenko2019, MacKinlay2021}.

To address this gap, we propose a framework that integrates neural operator surrogates with Thompson sampling-based acquisition strategies \citep{Russo2014} to actively optimize objectives of the form:
\[
a^* \in \argmax_{\infun \in \infspace} \objective(\nnop_*(\infun)),
\]
where \(\nnop_*: \infspace \to \outfspace\) is an unknown operator between function spaces \(\infspace\) and \(\outfspace\), and \(\objective:\outfspace \to \R\) is a known functional. We follow the steps of Bayesian optimization frameworks for composite functions \citep{Astudillo2019composite, Guilhoto2024neon}, which leverage knowledge of the composite structure to speed-up optimization, extending these frameworks to functional domains. Applying the theoretical results for the infinite-width limit of neural networks \citep{Jacot2018, Lee2019}, we show that a trained neural operator approximates a posterior sample from a vector-valued Gaussian process \citep{Rasmussen2006, Alvarez2012, Jorgensen2024} in a sample-then-optimize approach \citep{Dai2022sto}. Therefore, we are able to implement an approximate form of Thompson sampling without the need for expensive uncertainty quantification frameworks for neural operators, such as deep ensembles \citep{Pickering2022} or mixture density networks \citep{Li2024mrafno}, and derive theoretical regret bounds on its performance. Experiments evaluate our approach on problems with classic PDE benchmarks against Bayesian optimization baselines.

\section{Related work}

\paragraph{Bayesian optimization with functionals and operators.}
Bayesian optimization (BO) has been a successful approach for optimization problems involving expensive-to-evaluate black-box functions \citep{Shahriari2016}. Prior work on BO in function spaces includes Bayesian Functional Optimization (BFO) \citep{Vien2018}, which uses Gaussian processes to model objectives defined over functions, focusing on scalar functionals without explicitly learning operators. Follow-up work extended the framework to include prior information about the structure of the admissible input functions \citep{Vellanki2019bfosp}. \citet{Astudillo2019composite} introduced the framework of composite Bayesian optimization, which was later applied by \citet{Guilhoto2024neon} to optimization problems involving mappings from \emph{finite-dimensional} inputs to \emph{function-valued} outputs. Their objective was to optimize a known functional of these function-valued outputs.  Our approach differs by directly working in function spaces, involving function-to-function operators.  Despite the availability of GP models for function-to-function mappings \citep{Mora2025}, we are unaware of BO or GP-based bandit algorithms incorporating such models. Lastly, in the bandits literature, \citet{Tran-Thanh2014} introduced the problem of functional bandits. Despite the terminology, they deal with the problem of optimizing a known functional of the arms rewards \emph{distribution}, similar to the setting of distributionally robust BO \citep{Kirschner2020}, and therefore not directly comparable to our case.

\paragraph{Thompson sampling with neural networks.}
Neural Thompson Sampling (NTS) \citep{Zhang2021nts} employs neural networks trained via random initialization and gradient descent to approximate posterior distributions for bandit problems with scalar inputs and outputs, inspiring our use of randomized neural training for operator posterior sampling. The Sample-Then-Optimize Batch NTS (STO-BNTS) variant \citep{Dai2022sto} refines this by defining acquisition functions on functionals of posterior samples, facilitating composite objective optimization. STO-BNTS extends this to batch settings using Neural Tangent Kernel (NTK) and Gaussian process surrogates, relevant for future batched active learning with neural operators. These approaches rely on the NTK theory \citep{Jacot2018}, which shows that infinitely wide neural networks trained via gradient descent behave as Gaussian processes. To the best of our knowledge, this approach has not yet been extended to the case of neural network models with function-valued inputs, such as neural operators.

\paragraph{Active learning for neural operators.}

\citet{Pickering2022} applied deep operator networks (DeepONets) \citep{Lu2021deeponet} to the problem of Bayesian experimental design \citep{Rainforth2024}. In that framework, the goal is to select informative inputs (or designs) to reduce uncertainty about an unknown operator. To quantify uncertainty, \citet{Pickering2022} used an ensemble of DeepONets and quantified uncertainty in their predictions based on the variance of the ensemble outputs. \citet{Li2024mrafno} introduced multi-resolution active learning with Gaussian mixture models derived from Fourier neural operators \citep{Li2021fno}. With probabilistic outputs, mutual information can be directly quantified for active learning and Bayesian experimental design approaches. Lastly, \citet{Musekamp2025} proposed a benchmark for neural operator active learning and evaluated ensemble-based models with variance-based uncertainty quantification on tasks involving forecasting. In contrast to our focus in this paper, active learning approaches are purely focused on uncertainty reduction, neglecting other optimization objectives.

\section{Preliminaries}
\label{sec:setup}

\paragraph{Problem formulation.} Let \(\infspace\) and \(\outfspace\) denote two function spaces, and let \(\nnop_* : \infspace \to \outfspace\) be an unknown target operator\footnote{Here, we use the term \emph{unknown} loosely, in the sense that it is not fully implementable within the computational resources or paradigms accessible to us. For example, the target operator can be a simulator in a high-performance computing facility which we have limited access to.} between them. Consider an objective functional \(\objective : \outfspace \to \R\), which is assumed known and cheap to evaluate. Given a compact search space $\Sspace\subset \infspace$, we aim to solve:\footnote{We use ``$\in \argmax$'' acknowledging that the problem may have multiple global optima, forming a set of global optimizers. Whenever we assume a unique minimizer, we will use the equality symbol ``$=$'', instead.}
\begin{equation}
    \infun^* \in \argmax_{\infun \in \Sspace} \objective(\nnop_*(\infun)),
    \label{eq:problem}
\end{equation}
while \(\nnop_*\) is only accessible via expensive oracle queries: for a chosen \(\infun\), we observe a function-valued output \(\observation = \obsop\nnop_*(\infun) + \vnoise\), where $\obsop:\outfspace\to\obsspace$ represents an observation operator, typically the discretization on a grid, with $\obsspace$ being a (usually finite-dimensional) Hilbert space, and \(\vnoise \sim \normal(0, \vnCov)\) is observation noise, assumed
independent and identically distributed (i.i.d.) across queries.
The algorithm is allowed to query the oracle with any function $\infun\in\Sspace$ for up to a budget of $\nObs$ queries. For this paper, we focus on problems with finite search space $|\Sspace| < \infty$, though the framework is general.

\paragraph{Neural operators.} A neural operator is a specialized neural network architecture modeling  operators $\nnop:\infspace\to\outfspace$ between function spaces $\infspace$ and $\outfspace$ \citep{Kovachki2023}. Assume $\infspace \subset \cspace(\domain, \R^\indim)$ and $\outfspace \subset \cspace(\outdomain,\R^\outdim)$, where $\cspace(\set{S},\set{S}')$ denotes the space of continuous functions between sets $\set{S}$ and $\set{S}'$. Given an input function $\infun\in\infspace$, a neural operator $\nnop_\parameters$ performs a sequence of transformations $\infun =: \outfun_1\mapsto\cdots\mapsto\outfun_{\nlayers-1}\mapsto\outfun_{\nlayers}$ through $L$ layers of neural networks, where $\outfun_\layer: \domain_\layer \to \R^{\dimension_\layer}$ is a continuous function for each layer $\layer\in\{1,\dots,\nlayers\}$, and $\domain_\nlayers := \outdomain$ is the domain of the output functions and $\dimension_\nlayers := \outdim$. In one of its general formulations, for a given layer $\layer\in\{1,\dots,\nlayers\}$, the result of the transform (or update) at any $\location\in\domain_{\layer+1}$ can be described as:
\begin{equation}
    \begin{split}
        \outfun_1(\location) &:= \infun(\location) \\
        \outfun_{\layer+1}(\location) &:=
            \actfun_\layer \left( \int_{\domain_\layer}
            \nnkernel_\layer(
                \location, \location', \outfun_\layer(\nnproj_\layer(\location)), \outfun_\layer(\location')
                )\,
            \outfun_\layer(\location')\diff\lpmeasure_\layer(\location')
            + \Weights_\layer\,\outfun_\layer(\nnproj_\layer(\location))
            + \nnbias_\layer(\location) \right) \\
        \nnop_\parameters(\infun)(\outlocation) &:= \outfun_\nlayers(\outlocation)\,,
    \end{split}
    \label{eq:nnop}
\end{equation}
where $\nnproj_\layer:\domain_{\layer+1}\to\domain_\layer$ is a fixed mapping, $\actfun_\layer:\R\to\R$ denotes an activation function applied elementwise, $\nnkernel_\layer:\domain_{\iterIdx+1}\times\domain_\iterIdx\times\R^{\dimension_\layer}\times\R^{\dimension_\layer}\to\R^{\dimension_{\iterIdx+1}\times\dimension_{\iterIdx}}$ defines a (possibly nonlinear or positive-semidefinite) kernel integral operator with respect to a measure $\lpmeasure_\layer$ on $\domain_\layer$, $\Weights_\layer\in\R^{\dimension_{\layer+1}\times\dimension_{\layer}}$ is a weight matrix, and $\nnbias_\layer:\domain_{\layer+1}\to\R^{\dimension_{\layer+1}}$ is a bias function. We denote by $\parameters$ the collection of all learnable parameters of the neural operator: the weights matrices $\Weights_\layer$, the parameters of the bias functions $\nnbias_\layer$ and the matrix-valued kernels $\nnkernel_\layer$, for all layers $\layer\in\{1,\dots,\nlayers\}$.
Variations to the formulation above correspond to various neural operator architectures based on low-rank kernel approximations, graph structures, Fourier transforms, etc. \citep{Kovachki2023}.

\paragraph{Vector-valued Gaussian processes.} Vector-valued Gaussian processes extend scalar GPs \citep{Rasmussen2006} to the case of vector-valued functions \citep{Alvarez2012}. Let $\infspace$ be an arbitrary domain, and let $\outfspace$ be a Hilbert space representing a codomain. We consider the case where both the domain $\infspace$ and codomain $\outfspace$ might be infinite-dimensional vector spaces, which leads to GPs whose realizations are operators $\nnop_*:\infspace\to\outfspace$ \citep{Jorgensen2024}. To simplify our exposition, we assume that $\outfspace$ is a separable Hilbert space, though the theoretical framework is general enough to be extended to arbitrary Banach spaces \citep{Owhadi2019}. A vector-valued Gaussian process $\nnop_* \sim \gp(\widehat{\nnop}, \opkernel)$ on $\infspace$ is fully specified by a mean operator $\widehat{\nnop} : \infspace \to \outfspace$ and a positive-semidefinite operator-valued covariance function $\opkernel: \infspace \times \infspace \to \lopspace(\outfspace)$, where $\lopspace(\outfspace)$ denotes the space of bounded linear operators on $\outfspace$. Formally, given any $\infun, \infun' \in \infspace$ and any $\outfun,\outfun'\in\outfspace$, it follows that:
\begin{align}
    \expectation[\nnop_*(\infun)] &= \widehat{\nnop}(\infun), \\
    \covariance(\inner{\nnop_*(\infun), \outfun}, \inner{\nnop_*(\infun'), \outfun'}) &= \inner{\outfun, \opkernel(\infun, \infun')\outfun'}\,,
\end{align}
where $\inner{\cdot,\cdot}$ denotes the inner product and $\covariance(\cdot,\cdot)$ stands for the covariance between scalar variables.  Assume we are given a set of observations $\dataset_\iterIdx := \{(\infun_i, \observation_i)\}_{i=1}^\iterIdx \subset \infspace \times \outfspace$, where $\observation_i = \nnop_*(\infun_i) + \vnoise_i$, and $\vnoise_i \sim \normal(0, \vnCov)$ corresponds to Gaussian noise with covariance operator $\vnCov\in \lopspace(\outfspace)$. The posterior mean and covariance can then be defined by the following recursive relations:
\begin{align}
    \widehat{\nnop}_\iterIdx(\infun) &= \widehat{\nnop}_{\iterIdx-1}(\infun) + \opkernel_{\iterIdx-1}(\infun, \infun_{\iterIdx}) (\opkernel_{\iterIdx-1}(\infun_\iterIdx, \infun_\iterIdx) + \vnCov)^{-1} (\observation_\iterIdx -\widehat{\nnop}_{\iterIdx-1}(\infun_\iterIdx))\label{eq:gpop-mean}\\
    \opkernel_\iterIdx(\infun,\infun') &= \opkernel_{\iterIdx-1}(\infun,\infun') - \opkernel_{\iterIdx-1}(\infun,\infun_\iterIdx) (\opkernel_{\iterIdx-1}(\infun_\iterIdx, \infun_\iterIdx) + \vnCov)^{-1} \opkernel_{\iterIdx-1}(\infun_\iterIdx, \infun') \label{eq:gpop-cov}
\end{align}
for any $\infun, \infun' \in \infspace$, and $\iterIdx\in\N$, which are an extension of the same recursions from the scalar-valued case \citep[App. F]{Chowdhury2017} to the case of vector-valued processes. Such definition arises from sequentially conditioning the GP posterior on each observation, starting from the prior with $\widehat{\nnop}_0 := \widehat{\nnop}$ and $\opkernel_0 := \opkernel$.
This recursion leads to the same matrix-based definitions of the usual GP posterior equations \citep{Rasmussen2006}, but in our case it avoids complications with the resulting higher-order tensors that arise when kernels are operator-valued.

\begin{figure}[t]
\noindent
\begin{minipage}[t]{0.48\textwidth}
    \vspace{0pt}
    \begin{algorithm}[H]
        \caption{GP-TS}
        \label{alg:gpts}
        \DontPrintSemicolon
        \KwIn{Search space $\Sspace$, initial data \(\dataset_0\)}
        \For{$\iterIdx \in \{1,\dots, \nIterations\}$}{
            Sample $\gpfunction_\iterIdx \sim \gp(\gpmean_{\iterIdx-1}, \kernel_{\iterIdx-1})$\;
            Select $\location_\iterIdx \in \argmax_{\location\in\domain} \gpfunction_\iterIdx(\location)$\;
            Query $\observation_\iterIdx = \objective(\location_\iterIdx) + \obsnoise_\iterIdx$\;
            Update $\dataset_\iterIdx = \dataset_{\iterIdx-1} \cup \{\location_\iterIdx, \observation_\iterIdx\}$\;
        }
    \end{algorithm}
\end{minipage}%
\hfill 
\begin{minipage}[t]{0.48\textwidth}
    \vspace{0pt}
    \begin{algorithm}[H]
    \caption{NOTS (ours)}
    \label{alg:nots}
    \DontPrintSemicolon
    \KwIn{Search space $\Sspace$, initial data \(\dataset_0\)}
    \For{\(\iterIdx = 1, \ldots, \nIterations\)}{
        \(\parameters_\iterIdx = \argmin_{\parameters} \loss_\iterIdx(\parameters), \:\: \parameters_{\iterIdx, 0} \sim \normal(\vec 0, \mat\Sigma_0)\)\;
        \(\infun_\iterIdx \in \argmax_{\infun \in \Sspace} \objective(\nnop_{\parameters_\iterIdx}(\infun))\)\;
        \(\observation_\iterIdx = \nnop_*(\infun_\iterIdx) + \vnoise_\iterIdx \)\;
        \(\dataset_\iterIdx = \dataset_{\iterIdx-1} \cup \{\infun_\iterIdx, \observation_\iterIdx\}\)\;
    }
    \end{algorithm}
\end{minipage}    
\end{figure}

\paragraph{Thompson sampling.} Thompson sampling (TS) is a relatively simple randomized strategy for sequential decision making under uncertainty, which has found many successes in the Bayesian optimization and multi-armed bandits literature \citep{Russo2014, Kandasamy2018, Zhang2021nts, Takeno2024}. When applied to optimization problems, the core idea of TS is to query an objective function $\objective$ at points $\location_\iterIdx$ sampled from the probability distribution of the optimum location $\location^* \in \argmax_{\location\in\domain} \objective(\location)$ given the observations $\dataset_{\iterIdx-1} := \{\location_i, \observation_i\}_{i=1}^{\iterIdx-1}$. To do so, the objective function is modeled as sample from a Bayesian probabilistic model, which is typically a linear model \citep{Russo2014} or a GP \citep{Takeno2024}, and then TS samples realizations $\gpfunction_\iterIdx$ of the objective from the model's posterior $p(\objective|\dataset_{\iterIdx-1})$. A point $\location_\iterIdx$ which maximizes a sampled function $\gpfunction_\iterIdx$ then corresponds to a sample from the posterior distribution over the optimum $p(\location^* | \dataset_{\iterIdx-1})$. The procedure is summarized in \autoref{alg:gpts} for the case of a GP. Under mild assumptions, TS is known to produce a sequence of candidates $\location_\iterIdx$ such that $\objective(\location_\iterIdx)$ asymptotically converges to $\objective(\location^*)$ \citep{Russo2016,Takeno2024}.

\section{Neural operator Thompson sampling}
\label{sec:method}
We propose a Thompson sampling algorithm for the optimization of functionals of unknown operators in the setting of \autoref{eq:problem}. Instead of relying on extensions of traditional probabilistic methods to operator modeling, our method applies flexible and scalable neural operators as surrogates $\nnop_\iterIdx$, training them to approximate posterior samples over the true operator $\nnop_*$ conditioned on data. The method is designed to efficiently explore the search space while balancing the exploration-exploitation trade-off.

\subsection{Approximate posterior sampling}
\label{sec:training}
Given data \(\dataset_\iterIdx = \{(\infun_i, \observation_i)\}_{i=1}^\iterIdx\), we train a neural operator \(\nnop_\parameters\) with parameters \(\parameters_\iterIdx\) that minimize:
\begin{equation}
\loss_\iterIdx(\parameters) := \sum_{j=1}^{\iterIdx-1} \norm{\observation_j - \obsop\nnop_\parameters(\infun_j)}^2 + \regFactor \norm{\parameters}^2,
\label{eq:training}
\end{equation}
where \(\norm{\cdot}\) represents the norm in the underlying vector space, and
\(\regFactor > 0\) is a regularization factor which relates to the noise process $\vnoise$ \citep{Ordonez2025observation}.
We minimize $\loss_\iterIdx$ via gradient descent starting from $\parameters_{\iterIdx, 0} \sim \normal(0, \mat\Sigma_0)$, where $\mat\Sigma_0$ is a diagonal matrix following Kaiming He \citep{He2015nninit} or LeCun initialization \citep{LeCun1998}, which scale each layer's weights initialization variance according to the width of the previous layer. 
By an extension of standard results on the infinite-width limit of neural networks to the neural operator setting, we can show that the trained neural operator approximates a posterior sample from a vector-valued GP when, e.g., we train only the last linear layer (see App. \ref{sec:sampling-by-gd}), which in turn guarantees regret bounds (\autoref{sec:theory}). The prior over $\nnop_*$ is implicitly defined as the vector-valued Gaussian process given by the conjugate kernel \citep{Lee2018,Hu2021} associated with the neural operator architecture and the weights initialization distribution. Lastly, we note that, in practice, observations are discretized over a finite grid or other finite-dimensional representation \citep{Kovachki2023}, so that the observation space is $\obsspace\subseteq\R^\obsdim$ and the difference norms in \autoref{eq:training} reduce to Euclidean distances.

\subsection{Thompson sampling algorithm}
\label{sec:algorithm}
In \autoref{alg:nots}, we present the Neural Operator Thompson Sampling (NOTS) algorithm for the optimization of problem-dependent functionals of black-box operators. The algorithm operates sequentially over \(\nIterations\) iterations similar to standard GP-TS (\autoref{alg:gpts}). To sample a realization from the neural operator posterior, each iteration begins with the random initialization of the parameters of a neural operator that serves as a surrogate model for the true unknown operator. At each iteration, the neural operator model is trained according to \autoref{sec:training}, minimizing a regularized least-squares loss based on the currently available data, yielding an approximate sample $\nnop_\iterIdx:=\nnop_{\parameters_\iterIdx}$ from the true operator posterior $p(\nnop_*|\dataset_{\iterIdx-1})$.
The next step involves selecting the input for querying the oracle by maximizing the value of the objective functional $\objective$ over the neural operator's predictions $\nnop_\iterIdx(\infun)$. Finally, the algorithm runs the potentially expensive step of querying the true operator $\nnop_*$ with the selected input function $\infun_\iterIdx$, which may involve a complex simulation or physical experiment, and updates the dataset with the new (noisy) observation $\observation_\iterIdx$. This process repeats for up to $\nIterations$ iterations, producing a sequence of function-valued queries $\infun_\iterIdx$ that approximates the true optimum $\infun^*$ \eqref{eq:problem}.

\paragraph{Computational cost.} Each iteration of NOTS incurs a linear computational cost of $\bigo(\iterIdx)$ due to the retraining of the neural operator model, which can be further reduced by use of minibatch stochastic gradient descent. The reinitialization with randomized weights followed by retraining is what ensures that we have a new approximate posterior sample for TS conditioned on the available data at every iteration. Compared to a more traditional GP-based approach, which applied to our setting would incur a $\bigo(\iterIdx^3)$ cost per step due to the inversion of a covariance matrix of $\iterIdx$ data points, we achieve a much more computationally efficient and scalable algorithm, despite the cost of retraining the model.

\section{Theoretical results}
\label{sec:theory}

In this section, we establish the theoretical foundation of our proposed method. We show how a randomly initialized neural operator approximates a GP in the infinite-width limit through the use of the conjugate kernel, also known as the NNGP kernel \citep{Neal1996, Daniely2017, Lee2018, Fan2020, Hu2021}, under certain assumptions. This allows us to extend existing results for GP Thompson Sampling (GP-TS) \cite{Takeno2024} to our setting.

\subsection{Neural operator abstraction}
A neural operator models nonlinear operators \(\nnop:\infspace\to\outfspace\) between possibly infinite-dimensional function spaces \(\infspace\) and \(\outfspace\). Current results in NTK \citep{Jacot2018} and GP limits for neural networks \citep{Lee2019} do not immediately apply to this setting, as they rely on finite-dimensional domains. However, we can leverage an abstraction for neural operator architectures which sees their layers as maps over finite-dimensional inputs \citep{Nguyen2024}, which result from truncations to make the modeling problem tractable.

Considering a neural operator with a \emph{single} hidden layer, let $\paramdim\in\N$ represent the layer's width, $\nnkop:\infspace\to\cspace(\outdomain, \R^{\dimension_\nnkernel})$ denote a (fixed) continuous operator, and $\nnbias_0:\outdomain\to\R^{\dimension_\nnbias}$ denote a (fixed) continuous function.  For simplicity, we will assume scalar-valued output functions with $\outdim=1$. In general, with a single hidden layer, the model described in \autoref{eq:nnop} can be rewritten as:
\begin{equation}
    \nnop_\parameters(\infun)(\outlocation) = \weights_o^\transpose\actfun\left( \Weights_\nnkernel\nnkop(\infun)(\outlocation) + \Weights_\outfun\infun(\nnproj_0(\outlocation)) + \Weights_\nnbias\nnbias_0(\outlocation) \right)\,, \quad \outlocation\in\outdomain\,,
    \label{eq:nnop-abstraction}
\end{equation}
where $\parameters := \operatorname{vec}(\weights_o, \Weights_\nnkernel, \Weights_\outfun, \Weights_\nnbias) \in \R^{\paramdim(1 + \dimension_\nnkernel + \indim + \dimension_\nnbias)} =: \paramspace$ represents the model's flattened parameters. The finite weight matrix $\Weights_\nnkernel$ representing the kernel convolution integral arises as a result of truncations required in the practical implementation of neural operators (e.g., a finite number of Fourier modes or quadrature points).  With this formulation, one can recover most popular neural operator architectures \citep{Nguyen2024}. In the appendix, we discuss how Fourier neural operators \citep{Li2021fno} fit under this formulation, though the latter is general enough to incorporate other cases. We also highlight that neural operators possess universal approximation properties \citep{Kovachki2021universal}, given sufficient data and computational resources, despite the inherent low-rank approximations in their architecture.

\subsection{Infinite-width limit of neural operators}
\label{sec:nnop-limit}
With the construction in \autoref{eq:nnop-abstraction}, we can simply see the result of a neural operator layer when evaluated at a fixed $\outlocation\in\outdomain$ equivalently as a $\paramdim$-width feedforward neural network:
\begin{equation}
    \nnop_\parameters(\infun)(\outlocation) = \model_\parameters(\nninput_\outlocation(\infun)) := \weights_o^\transpose \actfun(\Weights\nninput_\outlocation(\infun))\,,
\end{equation}
where the input is given by $\nninput_\outlocation(\infun) := \left[ \nnkop(\infun)(\outlocation),\; \infun(\nnproj_0(\outlocation)),\; \nnbias_0(\outlocation) \right] \in \nndomain$, and $\nndomain := \R^{\dimension_\nnkernel + \indim + \dimension_\nnbias}$.

\paragraph{Conjugate kernel.} We can now derive infinite-width limits. The conjugate kernel describes the distribution of the untrained neural network $\model_\parameters:\nndomain\to\R$ under Gaussian weights initialization, whose infinite-width limit yields a Gaussian process \citep{Neal1996,Lee2018}. Formally, the conjugate kernel is defined as:
\begin{equation}
    \kernel_\model(\nninput, \nninput') := \lim_{\paramdim\to\infty} \expectation_{\parameters_0 \sim \normal(\vec 0, \paramcov_0)} [\model_{\parameters_0}(\nninput) \model_{\parameters_0}(\nninput')], \quad \nninput,\nninput'\in\nndomain\,.
\end{equation}
Since the composition of the map $\infspace\times\outdomain\ni(\infun,\outlocation)\mapsto\nninput_\outlocation(\infun)\in\nndomain$ with a kernel on $\nndomain$ yields a kernel on $\infspace\times\outdomain$ \citep[Lem. 4.3]{Steinwart2008}, the conjugate kernel of $\nnop_\parameters$ is determined by:
\begin{equation}
    \kernel_\nnop(\infun, \outlocation, \infun', \outlocation') := \kernel_\model(\nninput_\outlocation(\infun), \nninput_{\outlocation'}(\infun')), \quad \infun, \infun'\in\infspace, \quad \outlocation, \outlocation'\in\outdomain\,,
\end{equation}
where $\kernel_\model$ is the conjugate kernel of the neural network $\model_\parameters$.  Such a kernel defines a covariance function for a GP over the space of operators mapping $\infspace$ to $\outfspace$. Assume $\outfspace \subset \lpspace{2}(\lpmeasure)$ is a closed subspace of the space of functions which are square integrable with respect to a $\sigma$-finite Borel measure on $\outdomain$, and let $\lopspace(\outfspace)$ denote the space of linear operators on $\outfspace$. The following then defines a positive-semidefinite operator-valued kernel $\opkernel_\nnop:\infspace\times\infspace\to\lopspace(\outfspace)$:
\begin{equation}
    (\opkernel_\nnop(\infun, \infun')\outfun)(\outlocation) = \int_{\outdomain} \kernel_\nnop(\infun, \outlocation, \infun', \outlocation')\outfun(\outlocation')\diff\lpmeasure(\outlocation'),
    \label{eq:opkernel}
\end{equation}
for any $\outfun\in\outfspace$, $\infun, \infun'\in\infspace$ and $\outlocation\in\outdomain$. Hence, we can state the following result, whose proof can be found in Appendix \ref{app:infinite-width-limit}.

\begin{restatable}{proposition}{ckresult} 
    \label{thr:kernel}
    Let $\nnop_\parameters: \infspace \to \outfspace$ be a neural operator with a single hidden layer, where $\outfspace \subseteq \lpspace{2}(\lpmeasure)$ is closed, and $\lpmeasure$ is a finite Borel measure on $\outdomain$. Assume $\weights_o \sim \normal(\vec 0, \sigma_\parameters^2\eye)$, for $\sigma_\parameters^2 > 0$ such that $\sigma_\parameters^2 \propto 1/\paramdim$, while the remaining parameters have their entries sampled from a fixed normal distribution. Then, as $\paramdim\to\infty$, on every compact subset of $\infspace$, the neural operator converges in distribution to a zero-mean vector-valued Gaussian process with operator-valued covariance function given by:
        \begin{equation*}
            \lim_{\paramdim\to\infty} \expectation_{\parameters \sim \normal(\vec 0, \paramcov_0)} [\nnop_\parameters(\infun) \otimes \nnop_\parameters(\infun')] = \opkernel_\nnop(\infun,\infun')\,, \quad \infun, \infun'\in\infspace\,,
        \end{equation*}
        where $\opkernel_\nnop: \infspace \times \infspace \to \lopspace(\outfspace)$ is defined in \autoref{eq:opkernel}, and $\otimes$ denotes the outer product.
\end{restatable}

\subsection{Bayesian cumulative regret bounds}
\paragraph{Bayesian regret.} We analyze the performance of a sequential decision-making algorithm via its Bayesian cumulative regret. An algorithm's instant regret for querying $\infun_\iterIdx\in\infspace$ at iteration $\iterIdx\geq 1$ is:
\begin{equation}
    \regret_\iterIdx := \objective(\nnop_*(\infun^*)) - \objective(\nnop_*(\infun_\iterIdx))\,
\end{equation}
where $\infun^*$ is defined in \autoref{eq:problem}. 
The Bayesian cumulative regret after $\nIterations$ iterations is then defined as:
\begin{equation}
    \bcr_\nIterations := \expectation\left[\sum_{\iterIdx=1}^\nIterations \regret_\iterIdx\right]\,,
\end{equation}
where the expectation is over all sources of randomness affecting the decision-making process, i.e., the prior for $\nnop_*$ and the observation noise. If the algorithm achieves sub-linear cumulative regret, its simple regret asymptotically vanishes, as
$\lim_{\nIterations\to\infty} \expectation\left[\min_{\iterIdx\in\{1,\dots,\nIterations\}} \regret_\iterIdx \right] \leq \lim_{\nIterations\to\infty} \frac{1}{\nIterations}\Regret_\nIterations$, leading the algorithm's queries $\infun_\iterIdx$ to eventually approach the true optimum $\infun^*$.

\paragraph{Regularity assumptions.} For our analysis, we assume $\outfspace \subseteq \lpspace{2}(\lpmeasure)$ is a closed subspace of the Hilbert space $\lpspace{2}(\lpmeasure)$ of square-integrable $\lpmeasure$-measurable functions, for a given finite Borel measure $\lpmeasure$ on a compact domain $\outdomain$. We will assume the search space $\Sspace\subset\infspace$ is finite. The true operator $\nnop_*:\infspace\to\outfspace$ will be assumed to be a sample from a vector-valued Gaussian process $\nnop_* \sim \gp(0, \opkernel)$, where the operator-valued kernel $\opkernel: \infspace\times\infspace\to\lopspace(\outfspace)$ is given by the neural operator's infinite-width limit in \autoref{thr:kernel}. Observations $\observation = \obsop\nnop_*(\infun) + \vnoise$ are assumed to be corrupted by i.i.d. zero-mean Gaussian noise, $\vnoise \sim \normal(0, \vnCov)$, where $\vnCov$ is positive definite on $\obsspace\subseteq\R^\obsdim$.

We adapt state-of-the-art regret bounds for GP-TS \citep{Takeno2024} to an exact version of NOTS.  To do so, we first observe that, for a linear functional $\objective \in \lopspace(\outfspace, \R)$, the composition with a Gaussian random operator $\nnop_* \sim \gp(\widehat{\nnop}, \opkernel)$ yields a scalar-valued GP, i.e., $\objective \circ \nnop_* \sim \gp(\objective \circ \widehat{\nnop}, \objective^\transpose \opkernel \objective)$, where $\objective^\transpose \opkernel \objective:(\infun,\infun')\mapsto \objective (\opkernel(\infun,\infun')\objective)$. We can then extend GP-TS regret bounds to the case of operators.

\begin{restatable}{proposition}{regretbound}
    \label{thm:regretbound}
    Let $\objective:\outfspace\to\R$ be a bounded linear functional such that $\objective = \tilde{\objective} \circ \obsop$, where $\tilde{\objective}:\obsspace\to\R$ is linear, and $\nnop_*\sim\gp(0,\opkernel)$. Consider a sequential algorithm selecting $\infun_\iterIdx \in \argmax_{\infun\in\Sspace} \objective(\nnop_\iterIdx(\infun))$ and observing $\observation_\iterIdx = \obsop\nnop_*(\infun_\iterIdx) + \vnoise_\iterIdx$, where $\nnop_\iterIdx \overset{d}{=} \nnop_*|\dataset_\iterIdx$, and $\vnoise_\iterIdx \sim \normal(0, \regFactor\idop)$, for $\iterIdx\in\{1,\dots,\nIterations\}$. Then, this algorithm's expected cumulative regret is such that:
    \begin{equation}
        \bcr_\nIterations \in \bigo(\sqrt{\nIterations\mig_{\objective,\nIterations}}),
    \end{equation}
    where $\mig_{\objective, \nIterations}$ denotes the maximum information gain for a GP  with kernel $\kernel_{\objective} := \objective^\transpose \opkernel \objective$ and given $\nIterations$ observations. 
\end{restatable}

This result shows that NOTS can achieve sublinear cumulative regret in the infinite-width limit with an exact GP posterior sample. The result connects existing GP-TS guarantees to NOTS, and it differs from existing guarantees for other neural network based Thompson sampling algorithms \citep{Zhang2021nts, Dai2022sto}, which explored the scalar case and a frequentist setting (i.e., the objective function being a fixed element of the reproducing kernel Hilbert space defined by the network's neural tangent kernel). In the Bayesian setting, there is also no need for a time-dependent regularization parameter \citep{Dai2022sto}, allowing for a simpler implementation. Yet we note that \autoref{thm:regretbound} concerns the exact GP case. However, \autoref{thr:kernel} ensures that a single-hidden-layer randomly initialized neural operator follows a GP in the infinite-width limit, and we show in the appendix that training the last layer via gradient descent approximates a posterior sample, as in previous results for conventional neural networks \citep[App. D]{Lee2019}. \autoref{app:theory} presents proofs and further discussions on limitations and extensions, and a validation experiment can be found in \autoref{app:extra}.

\section{Experiments}

We evaluate our NOTS algorithm on two popular PDE benchmark problems: Darcy flow and a shallow water model. Our results are compared against a series of representative Bayesian optimization and neural Thompson sampling baselines. More details about our implementations and further experiment details can be found in \autoref{app:experiments}. Code for our experiments will be made available online.\footnote{Code repository: \url{https://github.com/csiro-funml/nots}}

\subsection{Algorithms}
We compare NOTS against a series of GP-based and neural network BO algorithms modeling directly the mapping from function-valued inputs $\infun\in\infspace$ (discretized over regular grid) to the scalar-valued functional evaluations $\objective(\nnop_*(\infun))$, besides a trivial random search (RS) baseline. NOTS is implemented with standard and spherical FNOs \citep{bonev2023spherical}, following default library settings for these PDEs \citep{kossaifi2024neural}. We first implemented BO with a 3-layer infinite-width ReLU Bayesian neural network (BNN) model, represented as a GP with the corresponding conjugate kernel. According to \citet{Li2024bnnbo}, these models can achieve optimal performance in high-dimensional settings when compared to other BNN methods. Two versions of this framework are in our experiments, one with log-expected improvement, given its well established competitive performance \citep{Ament2023}, simply denoted as ``BO'' in our plots, and one with Thompson sampling (GP-TS) \citep{Russo2016}. As our experiments are over finite domains, sampling from a scalar GP boils down to sampling from a multivariate normal distribution. Next, we evaluated a version of Bayesian functional optimization (BFO) by encoding input functions in a reproducing kernel Hilbert space (RKHS) via their minimum-norm interpolant and using a squared-exponential kernel over functions which takes advantage of the RKHS structure as in the original BFO \citep{Vien2018}. Lastly, we evaluated sample-then-optimize neural Thompson sampling (STO-NTS), training a 2-layer 256-width fully connected neural network with a regularized least-squares loss \citep{Dai2022sto}. 

\subsection{PDE benchmarks}
\label{sec:benchmarks}
\paragraph{Darcy flow.}
Darcy flow models fluid pressure in a porous medium \citep{Li2021fno}, with applications in contaminant control, leakage reduction, and filtration design. In our setting, the input $a\in\mathcal{C}((0,1)^2,\mathbb{R}_+)$ is the medium’s permeability on a Dirichlet boundary, and the operator $G_\star$ maps $a$ to the pressure field $u\in\mathcal{C}((0,1)^2,\mathbb{R})$. To train $G_\theta$, we generate 1,000 input–output pairs via a finite-difference solver at $16\times16$ resolution. Two materials are considered, leading to a binary grid for $\infun$ and a continuum of pressure values for each $u$ grid cell. More details are in \citet{Li2021fno} and \autoref{app:experiments}.

\paragraph{Shallow water modeling.}
Shallow water models capture the time evolution of fluid mass and discharge on a rotating sphere \citep{bonev2023spherical}. The input $a\in\mathcal{C}(\mathbb{S}^2\!\times\{t=0\},\mathbb{R}^3)$ represents the initial geopotential depth and two velocity components, while the output $u\in\mathcal{C}(\mathbb{S}^2\!\times\{t=\tau\},\mathbb{R}^3)$ gives the state at time $t=\tau$. We train $G_\theta$ on 200 random initial conditions on a $32\times64$ equiangular grid, using a 1,200 s timestep to simulate up to $\tau=$ 6 hours. 

\subsection{Optimization functionals}
We introduce several optimization functionals that are problem-dependent and clarify their physical meaning in the context of the benchmark problems. As we aim to solve a maximization problem, physical quantities to be minimized are defined with a negative sign. The first three functionals were applied to the Darcy flow problem and the last one to shallow water modeling. Note that in both cases, we have the same domain for the PDE solutions $\outfun$ and input functions $\infun$, i.e., $\outdomain=\domain$.

\paragraph{Negative total flow rates \citep{katz1979history} \(f(u, a) = -\int_{\partial\domain} \infun(\location) (\nabla \outfun(\location) \cdot n)d\location\).}  Here \(\partial\domain\) is the boundary of the domain and \(n\) is the outward pointing unit normal vector of the boundary. This functional integrates the volumetric flux \(-a(\location)\nabla u(\location)\) along the boundary, which corresponds to the total flow rate of the fluid. Such an objective can be optimized for leakage reduction and contaminant control.

\paragraph{Negative total pressure \citep{jeong2025optimal} \(f(u) = -\frac{1}{2} \int_\mathcal{\domain} |u(\location)| d\location\).} This objective computes the total fluid pressure over the domain in the Darcy flow system.

\paragraph{Negative total potential energy $\objective(\outfun, \infun) = -\int_\domain \infun(\location)\norm{\nabla\outfun(\location)}^2\diff\location + \int_\domain s(\location)\outfun(\location)\diff\location$.} This functional quantifies the system's total potential energy, balancing the energy dissipated by fluid friction (the first term) against the potential energy supplied by the uniform fluid source (the second term, where $s=1$ is assumed). The minimizer $\infun^*$, therefore, consists of the most hydrodynamically efficient design for the given flow constraints.

\paragraph{Inverse problem \(f(u) = -\frac{1}{2} \norm{ u-u_{\tau}}^2 \).} \(u_\tau\) represents the ground truth solution. This objective is specific to shallow water modeling, as we aim to find the initial condition \(a\) that generates \(u_\tau\) at time $\tau$, which is also a simplification of the assimilation objective in weather forecasting \citep{rabier1998extended, xiao2023fengwu}.

\begin{figure}[t]
    \centering
    \includegraphics[width=0.45\linewidth]{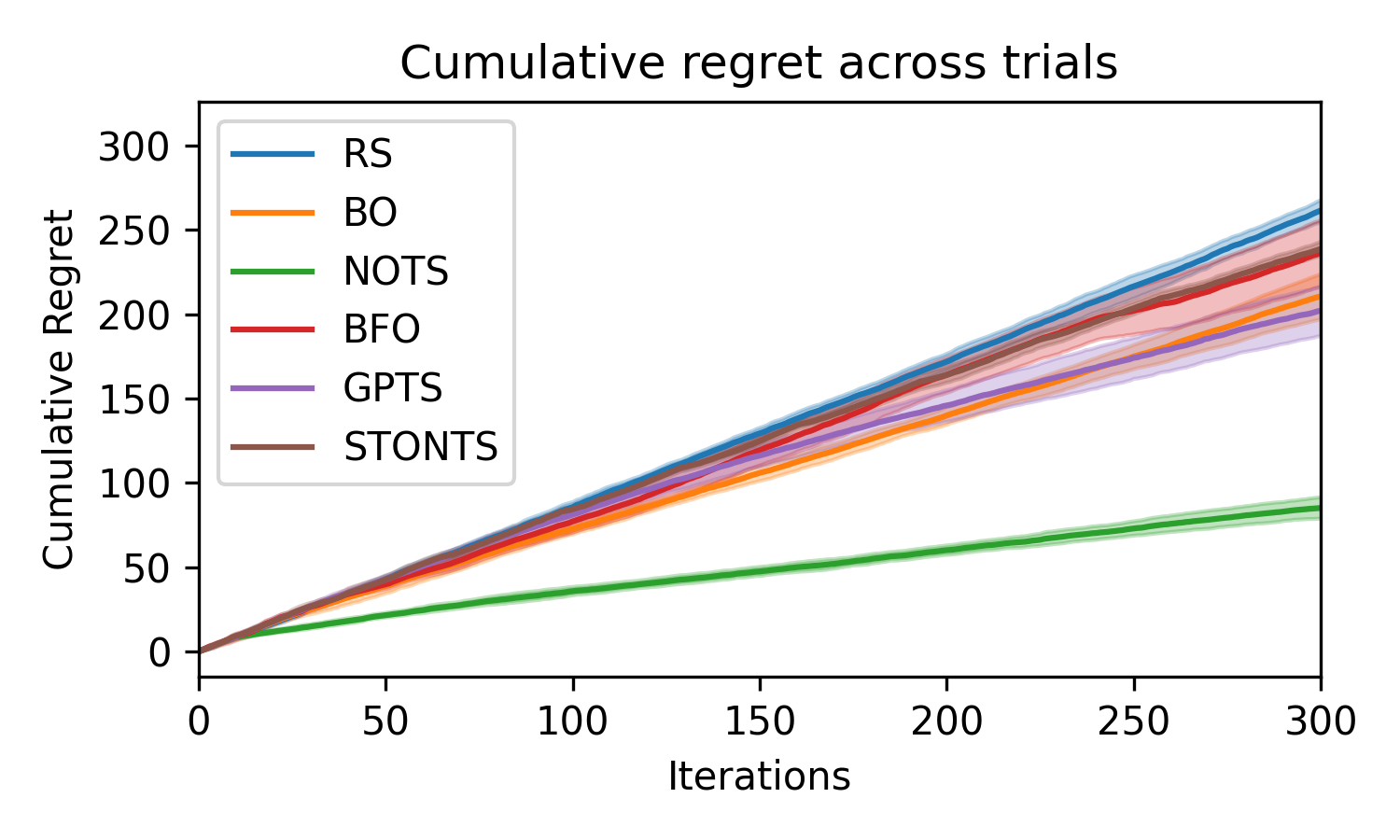}
    \hspace{0.02\linewidth}
    \includegraphics[width=0.45\linewidth]{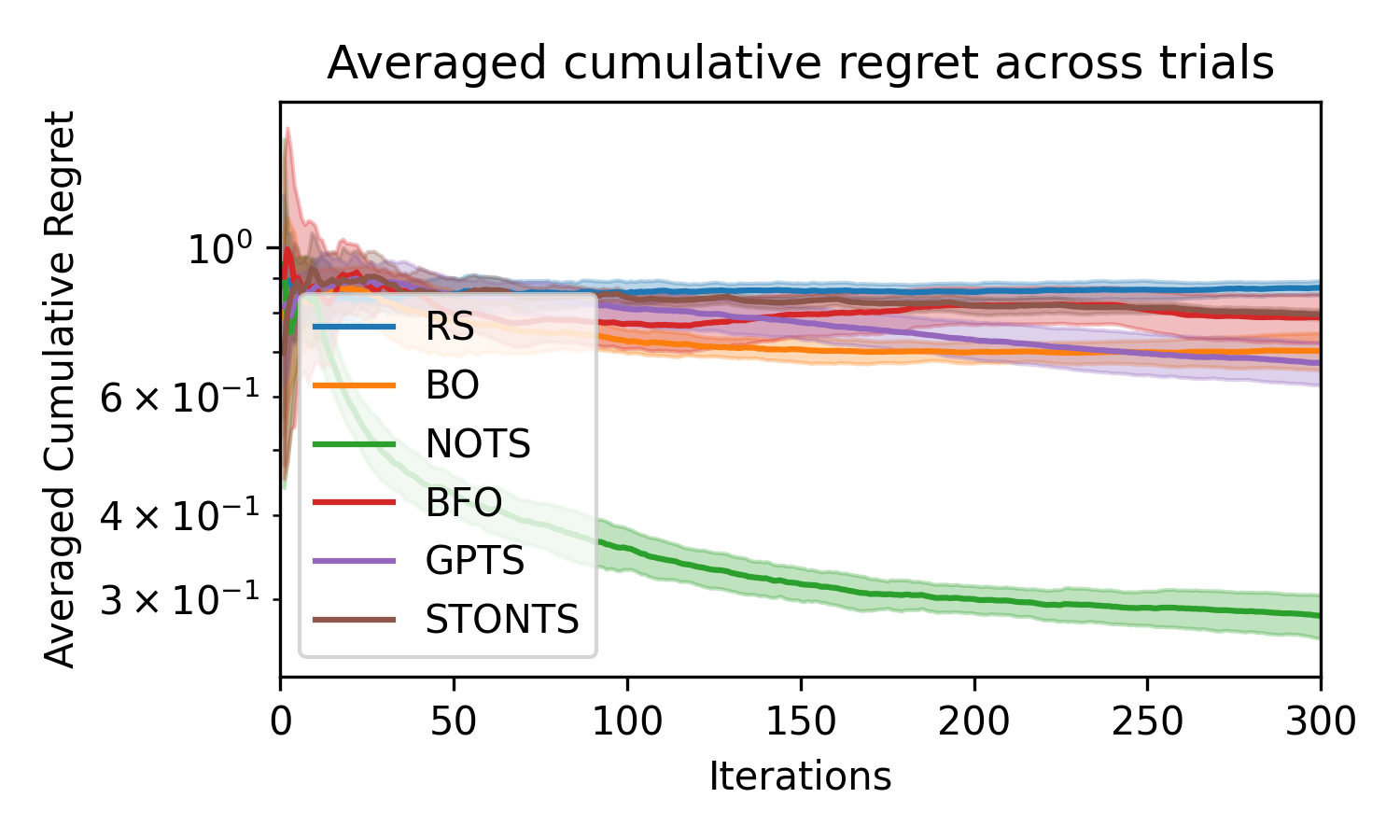}
    \hspace{0.02\linewidth}
    \begin{footnotesize}
    \begin{tabular}{@{\hskip 2em}p{7.6em}@{\hskip 3em}p{7.6em}@{\hskip 3em}p{7.6em}@{\hskip 3em}p{7.6em}}
    best candidate {input} function $a$&
    best candidate output function $u$&
    worst candidate input function $a$&
    worst candidate output function $u$
    \end{tabular}
    \end{footnotesize}
    \includegraphics[width=0.96\linewidth, trim={0 0 0 1.1cm},clip]{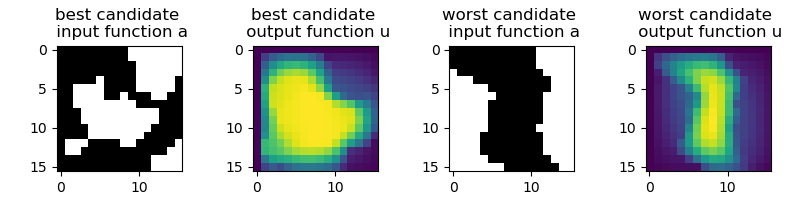}

    \caption{Darcy flow rate optimization. Overlay of cumulative regret (top left) and its average (top right) metrics across trials for the negative total flow rates case in the Darcy flow problem. The shaded areas correspond to one standard deviation across 10 trials. The corresponding input-output functions that achieved the best and worst flow rates are presented (bottom). White regions \(a(x)=1\) 
    means fully open permeability and black regions \(a(x)=0\) represents impermeable pore material. The output function suggests pressure field where brighter color indicates higher pressure.
    }
    \label{fig:regret-darcy-flow-rates}
\end{figure}

\begin{figure}[t]
    \centering
    \subfigure[Pressure]{\includegraphics[width=0.45\linewidth]{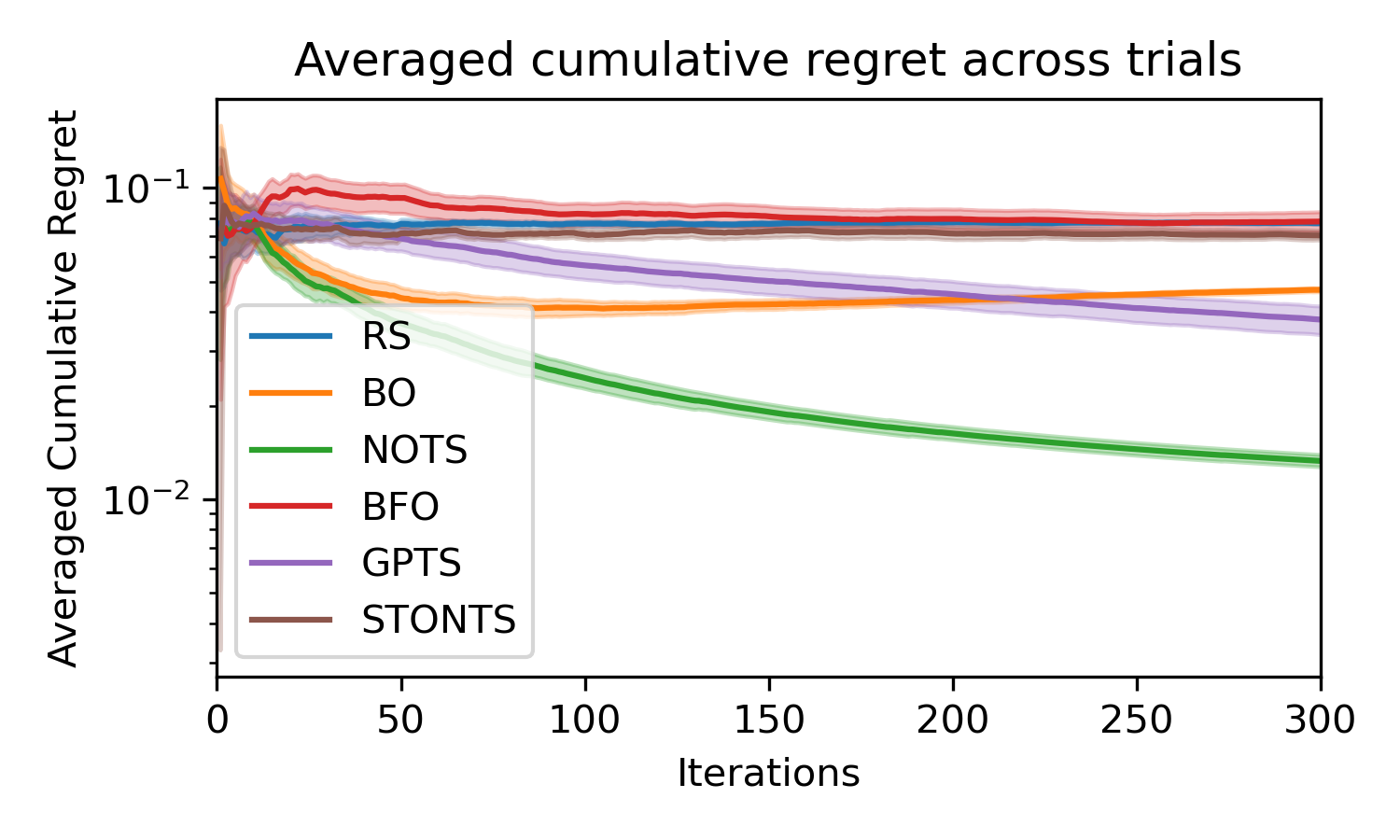} \label{fig:regret-darcy-flow-pressure}}
    \subfigure[Potential energy]{\includegraphics[width=0.45\linewidth]{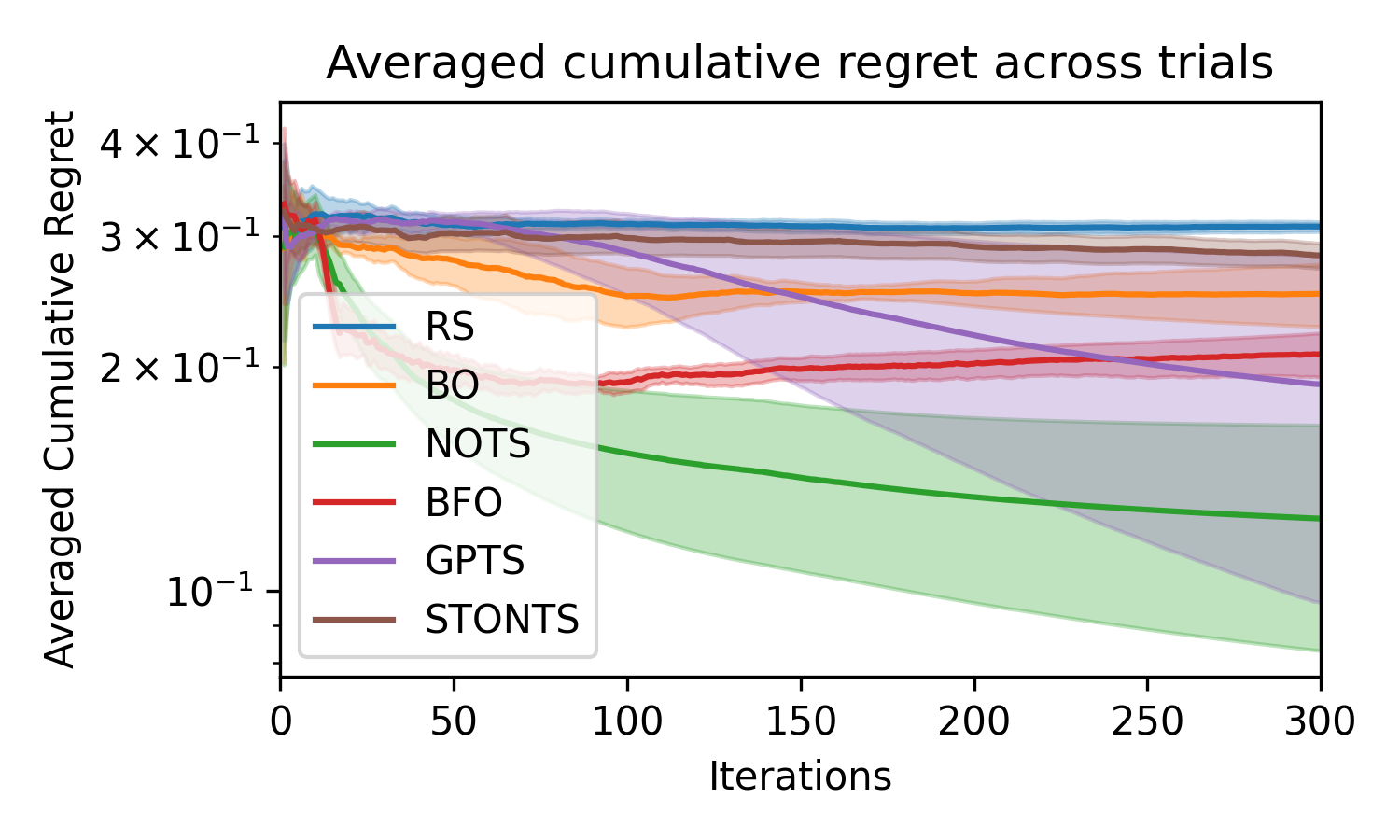} \label{fig:regret-darcy-flow-energy}}

    \caption{Darcy flow pressure \protect\subref{fig:regret-darcy-flow-pressure} and potential energy \protect\subref{fig:regret-darcy-flow-energy} optimization problems averaged cumulative regret. The shaded areas correspond to one standard deviation across 10 trials.}
    \label{fig:regret-darcy-flow-power-pressure}
\end{figure}

\begin{figure}[t]
    \centering
    \includegraphics[width=0.45\linewidth]{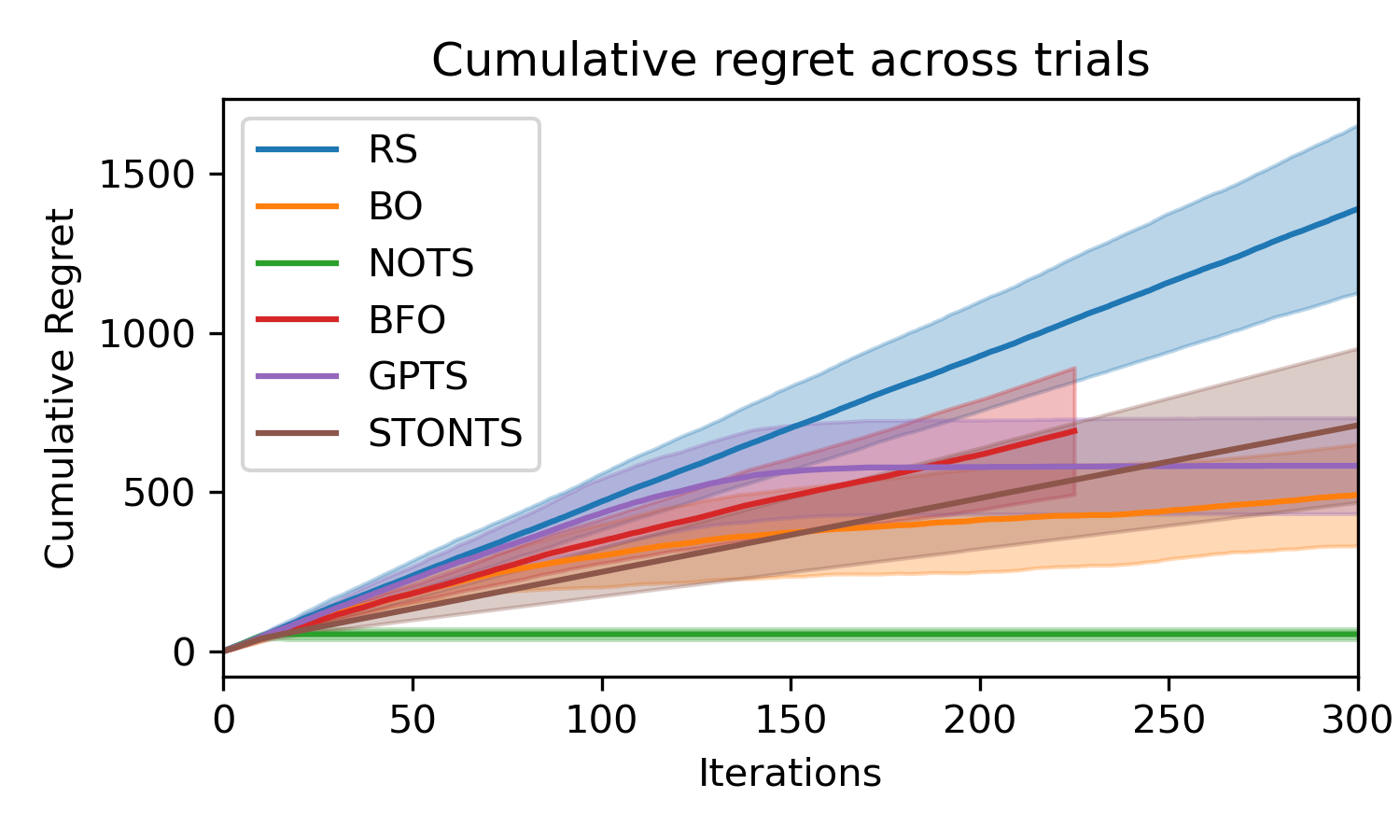}
    \hspace{0.02\linewidth}
    \includegraphics[width=0.45\linewidth]{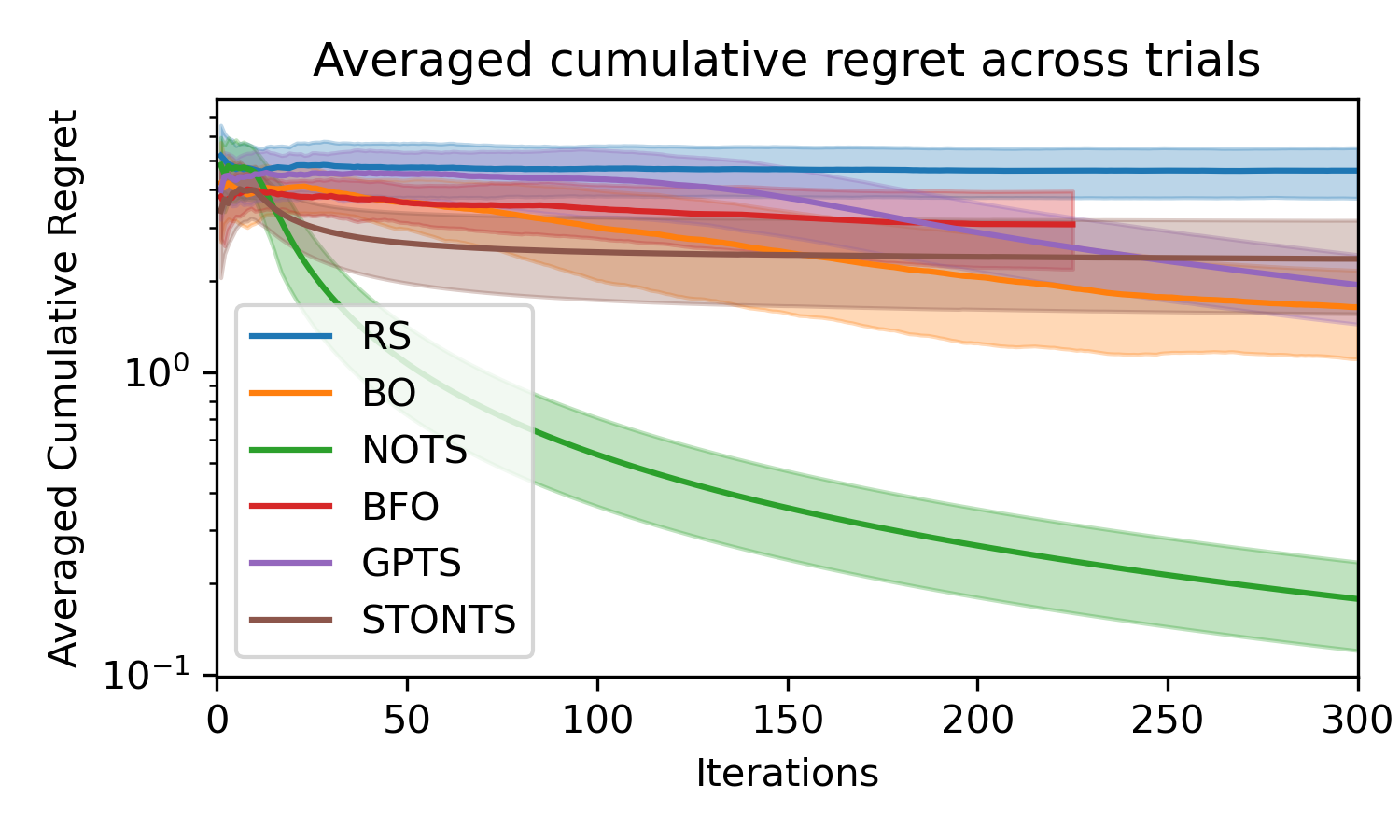}

    \caption{Shallow water inverse problem. Overlay of cumulative regret (left) and its average (right) metrics across trials for the inverse problem in the shallow water data. The shaded areas correspond to one standard deviation across 10 trials.}
    \label{fig:regret-swe-inverse}
\end{figure}

\subsection{Results}
Our results are presented in \autoref{fig:regret-darcy-flow-rates} to \ref{fig:regret-swe-inverse}, comparing the cumulative regret of NOTS against the baselines on different settings of PDE problems and functional objectives. Results are summarized in \autoref{tab:summary} with the final average regret, i.e., $\frac{\Regret_\nIterations}{\nIterations}$, of each method across the different problems.

\begin{table}[t]
\centering
\caption{Results summary: Final average regret of each method and its standard deviation.}
\label{tab:summary}
\begin{tabular}{lcccc}
\hline
Method      & Darcy flow rates       & Darcy flow energy   & Darcy flow pressure   & Shallow water \\
\hline
RS          & \(0.872 \pm 0.022\)      & \(0.309 \pm 0.005\)      & \(0.077 \pm 0.001\)      & \(4.632 \pm 0.876\)        \\
BO          & \(0.703 \pm 0.045\)      & \(0.251 \pm 0.024\)      & \(0.047 \pm 0.001\)      & \(1.639 \pm 0.532\)        \\
BFO         & \(0.788 \pm 0.066\)      & \(0.208 \pm 0.014\)      & \(0.078 \pm 0.006\)      & \(3.076 \pm 0.886\)        \\
GP-TS       & \(0.674 \pm 0.050\)      & \(0.189 \pm 0.093\)      & \(0.038 \pm 0.004\)      & \(1.942 \pm 0.502\)        \\
STO-NTS     & \(0.068 \pm 0.002\)      & \(0.282 \pm 0.011\)      & \(0.068 \pm 0.002\)      & \(2.329 \pm 0.800\)        \\
NOTS        & \(\mathbf{0.012 \pm 0.001}\)      & \(\mathbf{0.125 \pm 0.042}\)      & \(\mathbf{0.012 \pm 0.001}\)      & \(\mathbf{0.134 \pm 0.043}\)        \\
\hline
\end{tabular}
\end{table}

In \autoref{fig:regret-darcy-flow-rates}, we present our results for the flow rate optimization problem in the Darcy flow PDE benchmark. The results clearly show that GP-based BO methods struggle in this high-dimensional setting, while NOTS (ours) is able to consistently find optimal solutions. As described in \autoref{sec:benchmarks}, input functions $\infun\in\infspace$ for Darcy flow are binary masks representing two materials of different permeability which are discretized over a 2D grid of 16-by-16 sampling locations. Hence, when applied to standard GP-based BO methods, the inputs correspond to 256-dimensional vectors, which can be quite high-dimensional for standard GPs.
The optimization results of the input and output functions also show the effectiveness of our approach. In the case of the ``best candidate'' which achieves the lowest total flow rate, the input function shows large contiguous impermeable regions that block fluid outflow and thus generate high interior pressure which can be treated as an ideal design for leakage control. In contrast, the ``worst candidate'' exhibits the highest total flow rates. It has smooth, boundary-connected permeable zones allowing fluid to escape effortlessly. 
Lastly, figures \ref{fig:regret-darcy-flow-pressure} and \ref{fig:regret-darcy-flow-energy} show the results on optimizing pressure and potential energy on Darcy flow. On these functionals, BO and GP-TS can achieve a better performance, recalling their use of the infinite-width BNN kernel, which has shown good performance on high-dimensional problems \citep{Li2024bnnbo}. Yet, we can see significant performance improvements from NOTS with respect to all baselines.

\autoref{fig:regret-swe-inverse} shows our results for the inverse problem on the shallow water PDE benchmark. This setting involves higher dimensional discretized inputs (6144-dimensional when flattened), leading to an extremely challenging problem for GP approaches. In particular, the evaluation of the functional inputs kernel is too computationally intensive for BFO, leading it to crash before 250 iterations are completed. We believe that STO-NTS's low performance is due to architectural limitations, as it uses a simple fully connected network, which leads to a need for higher amounts of data (i.e., more iterations). NOTS, however, is able to learn the underlying physics of the problem to aid its predictions, leading to a more efficient exploration  and higher performance.

\section{Conclusion}
We have developed Neural operator Thompson sampling (NOTS) for optimization problems in function spaces and  shown that it provides significant performance gains in encoding the compositional structure of problems involving black-box operators, such as complex physics simulators or real physical processes. NOTS  also comes equipped with theoretical guarantees, connecting the existing literature on Thompson sampling to this novel setting involving neural operators.

\paragraph{Discussion.} We have shown empirically that using neural operators as surrogates for Thompson sampling can be effective without the need for expensive uncertainty quantification schemes by relying on theoretical results for infinitely wide deep neural networks and their connection with Gaussian processes. Neural operators have allowed for effective representation learning which scales to very high-dimensional settings, where traditional bandits and Bayesian optimization algorithms would struggle. Although GPs typically perform well on Bayesian modeling tasks with low volumes of data, the functional optimization problems we considered have high-dimensional data as both inputs and outputs, rendering the application of traditional multi-output GP models challenging. The basic computational complexity of inference with a vector-valued GP model scales cubically with both the number of data points and the number of output coordinates \citep{Alvarez2012}. For the shallow water PDE, for example, both inputs and outputs lie in a 6144-dimensional space. With 300 iterations, a multi-output GP would have to invert a kernel matrix over more than 1 million data points towards the last iterations. Hence, without specialized kernels and computationally efficient approximations, a traditional GP approach would be unsuitable due to the very large number of outputs.  In contrast, neural operators are specially designed to deal with function-valued input and output data, typically over spatial domains, with linearly scaling computational complexity. Therefore, NOTS can better scale to accommodate longer runs or extensions to batched evaluations than a GP approach, even though we limited experiments to 300 iterations to allow for comparisons against GP baselines.

\paragraph{Limitations and future work.} We note that our current results are focused on the case of finite search spaces and well specified models, which provide a first step towards more general use cases. An extension to continuous domain could, for example, parameterize the set of input functions and optimize such parametric representation or tractable nonparametric extensions \citep{Vien2018, Vellanki2019bfosp}, which might be application specific. Our theoretical analysis only considered the case of a neural operator with a single hidden layer, despite the multi-layer setting in our experiments. These and other limitations are further discussed in \autoref{app:limitations}. As future work, we plan to investigate the generalization of our results to more general settings, such as continuous domains and batched evaluations. Lastly, we note that NOTS also offers a framework for task-to-task amortization and few-shot learning, as operator learning data can be reused across tasks with different objective functionals.

\begin{ack}
    This research was carried out solely using CSIRO's resources. Chai contributed while on sabbatical leave visiting the Machine Learning and Data Science Unit at Okinawa Institute of Science and Technology, and the Department of Statistics in the University of Oxford. This project was supported by resources and expertise provided by CSIRO IMT Scientific Computing.
\end{ack}

\bibliographystyle{unsrtnat}
\bibliography{clean_refs}

\newpage
\section*{NeurIPS Paper Checklist}

\begin{enumerate}

\item {\bf Claims}
    \item[] Question: Do the main claims made in the abstract and introduction accurately reflect the paper's contributions and scope?
    \item[] Answer: \answerYes{} 
    \item[] Justification: Demonstrated by theoretical and experimental results.
    \item[] Guidelines:
    \begin{itemize}
        \item The answer NA means that the abstract and introduction do not include the claims made in the paper.
        \item The abstract and/or introduction should clearly state the claims made, including the contributions made in the paper and important assumptions and limitations. A No or NA answer to this question will not be perceived well by the reviewers. 
        \item The claims made should match theoretical and experimental results, and reflect how much the results can be expected to generalize to other settings. 
        \item It is fine to include aspirational goals as motivation as long as it is clear that these goals are not attained by the paper. 
    \end{itemize}

\item {\bf Limitations}
    \item[] Question: Does the paper discuss the limitations of the work performed by the authors?
    \item[] Answer: \answerYes{} 
    \item[] Justification: Discussion in the appendix and the conclusion
    \item[] Guidelines:
    \begin{itemize}
        \item The answer NA means that the paper has no limitation while the answer No means that the paper has limitations, but those are not discussed in the paper. 
        \item The authors are encouraged to create a separate "Limitations" section in their paper.
        \item The paper should point out any strong assumptions and how robust the results are to violations of these assumptions (e.g., independence assumptions, noiseless settings, model well-specification, asymptotic approximations only holding locally). The authors should reflect on how these assumptions might be violated in practice and what the implications would be.
        \item The authors should reflect on the scope of the claims made, e.g., if the approach was only tested on a few datasets or with a few runs. In general, empirical results often depend on implicit assumptions, which should be articulated.
        \item The authors should reflect on the factors that influence the performance of the approach. For example, a facial recognition algorithm may perform poorly when image resolution is low or images are taken in low lighting. Or a speech-to-text system might not be used reliably to provide closed captions for online lectures because it fails to handle technical jargon.
        \item The authors should discuss the computational efficiency of the proposed algorithms and how they scale with dataset size.
        \item If applicable, the authors should discuss possible limitations of their approach to address problems of privacy and fairness.
        \item While the authors might fear that complete honesty about limitations might be used by reviewers as grounds for rejection, a worse outcome might be that reviewers discover limitations that aren't acknowledged in the paper. The authors should use their best judgment and recognize that individual actions in favor of transparency play an important role in developing norms that preserve the integrity of the community. Reviewers will be specifically instructed to not penalize honesty concerning limitations.
    \end{itemize}

\item {\bf Theory assumptions and proofs}
    \item[] Question: For each theoretical result, does the paper provide the full set of assumptions and a complete (and correct) proof?
    \item[] Answer: \answerYes{}{} 
    \item[] Justification: In the appendix (supplementary material), the reader can find the proofs and full assumptions.
    \item[] Guidelines:
    \begin{itemize}
        \item The answer NA means that the paper does not include theoretical results. 
        \item All the theorems, formulas, and proofs in the paper should be numbered and cross-referenced.
        \item All assumptions should be clearly stated or referenced in the statement of any theorems.
        \item The proofs can either appear in the main paper or the supplemental material, but if they appear in the supplemental material, the authors are encouraged to provide a short proof sketch to provide intuition. 
        \item Inversely, any informal proof provided in the core of the paper should be complemented by formal proofs provided in appendix or supplemental material.
        \item Theorems and Lemmas that the proof relies upon should be properly referenced. 
    \end{itemize}

    \item {\bf Experimental result reproducibility}
    \item[] Question: Does the paper fully disclose all the information needed to reproduce the main experimental results of the paper to the extent that it affects the main claims and/or conclusions of the paper (regardless of whether the code and data are provided or not)?
    \item[] Answer: \answerYes{} 
    \item[] Justification: Details in the appendix.
    \item[] Guidelines:
    \begin{itemize}
        \item The answer NA means that the paper does not include experiments.
        \item If the paper includes experiments, a No answer to this question will not be perceived well by the reviewers: Making the paper reproducible is important, regardless of whether the code and data are provided or not.
        \item If the contribution is a dataset and/or model, the authors should describe the steps taken to make their results reproducible or verifiable. 
        \item Depending on the contribution, reproducibility can be accomplished in various ways. For example, if the contribution is a novel architecture, describing the architecture fully might suffice, or if the contribution is a specific model and empirical evaluation, it may be necessary to either make it possible for others to replicate the model with the same dataset, or provide access to the model. In general. releasing code and data is often one good way to accomplish this, but reproducibility can also be provided via detailed instructions for how to replicate the results, access to a hosted model (e.g., in the case of a large language model), releasing of a model checkpoint, or other means that are appropriate to the research performed.
        \item While NeurIPS does not require releasing code, the conference does require all submissions to provide some reasonable avenue for reproducibility, which may depend on the nature of the contribution. For example
        \begin{enumerate}
            \item If the contribution is primarily a new algorithm, the paper should make it clear how to reproduce that algorithm.
            \item If the contribution is primarily a new model architecture, the paper should describe the architecture clearly and fully.
            \item If the contribution is a new model (e.g., a large language model), then there should either be a way to access this model for reproducing the results or a way to reproduce the model (e.g., with an open-source dataset or instructions for how to construct the dataset).
            \item We recognize that reproducibility may be tricky in some cases, in which case authors are welcome to describe the particular way they provide for reproducibility. In the case of closed-source models, it may be that access to the model is limited in some way (e.g., to registered users), but it should be possible for other researchers to have some path to reproducing or verifying the results.
        \end{enumerate}
    \end{itemize}

\item {\bf Open access to data and code}
    \item[] Question: Does the paper provide open access to the data and code, with sufficient instructions to faithfully reproduce the main experimental results, as described in supplemental material?
    \item[] Answer: \answerYes{} 
    \item[] Justification: Code will be released at \url{https://github.com/csiro-funml/nots}.
    \item[] Guidelines:
    \begin{itemize}
        \item The answer NA means that paper does not include experiments requiring code.
        \item Please see the NeurIPS code and data submission guidelines (\url{https://nips.cc/public/guides/CodeSubmissionPolicy}) for more details.
        \item While we encourage the release of code and data, we understand that this might not be possible, so “No” is an acceptable answer. Papers cannot be rejected simply for not including code, unless this is central to the contribution (e.g., for a new open-source benchmark).
        \item The instructions should contain the exact command and environment needed to run to reproduce the results. See the NeurIPS code and data submission guidelines (\url{https://nips.cc/public/guides/CodeSubmissionPolicy}) for more details.
        \item The authors should provide instructions on data access and preparation, including how to access the raw data, preprocessed data, intermediate data, and generated data, etc.
        \item The authors should provide scripts to reproduce all experimental results for the new proposed method and baselines. If only a subset of experiments are reproducible, they should state which ones are omitted from the script and why.
        \item At submission time, to preserve anonymity, the authors should release anonymized versions (if applicable).
        \item Providing as much information as possible in supplemental material (appended to the paper) is recommended, but including URLs to data and code is permitted.
    \end{itemize}

\item {\bf Experimental setting/details}
    \item[] Question: Does the paper specify all the training and test details (e.g., data splits, hyperparameters, how they were chosen, type of optimizer, etc.) necessary to understand the results?
    \item[] Answer: \answerYes{} 
    \item[] Justification: In the appendix (supplement)
    \item[] Guidelines:
    \begin{itemize}
        \item The answer NA means that the paper does not include experiments.
        \item The experimental setting should be presented in the core of the paper to a level of detail that is necessary to appreciate the results and make sense of them.
        \item The full details can be provided either with the code, in appendix, or as supplemental material.
    \end{itemize}

\item {\bf Experiment statistical significance}
    \item[] Question: Does the paper report error bars suitably and correctly defined or other appropriate information about the statistical significance of the experiments?
    \item[] Answer: \answerYes{} 
    \item[] Justification: Standard deviations reported with the plots
    \item[] Guidelines:
    \begin{itemize}
        \item The answer NA means that the paper does not include experiments.
        \item The authors should answer "Yes" if the results are accompanied by error bars, confidence intervals, or statistical significance tests, at least for the experiments that support the main claims of the paper.
        \item The factors of variability that the error bars are capturing should be clearly stated (for example, train/test split, initialization, random drawing of some parameter, or overall run with given experimental conditions).
        \item The method for calculating the error bars should be explained (closed form formula, call to a library function, bootstrap, etc.)
        \item The assumptions made should be given (e.g., Normally distributed errors).
        \item It should be clear whether the error bar is the standard deviation or the standard error of the mean.
        \item It is OK to report 1-sigma error bars, but one should state it. The authors should preferably report a 2-sigma error bar than state that they have a 96\% CI, if the hypothesis of Normality of errors is not verified.
        \item For asymmetric distributions, the authors should be careful not to show in tables or figures symmetric error bars that would yield results that are out of range (e.g. negative error rates).
        \item If error bars are reported in tables or plots, The authors should explain in the text how they were calculated and reference the corresponding figures or tables in the text.
    \end{itemize}

\item {\bf Experiments compute resources}
    \item[] Question: For each experiment, does the paper provide sufficient information on the computer resources (type of compute workers, memory, time of execution) needed to reproduce the experiments?
    \item[] Answer: \answerYes{} 
    \item[] Justification: Details in the appendix
    \item[] Guidelines:
    \begin{itemize}
        \item The answer NA means that the paper does not include experiments.
        \item The paper should indicate the type of compute workers CPU or GPU, internal cluster, or cloud provider, including relevant memory and storage.
        \item The paper should provide the amount of compute required for each of the individual experimental runs as well as estimate the total compute. 
        \item The paper should disclose whether the full research project required more compute than the experiments reported in the paper (e.g., preliminary or failed experiments that didn't make it into the paper). 
    \end{itemize}
    
\item {\bf Code of ethics}
    \item[] Question: Does the research conducted in the paper conform, in every respect, with the NeurIPS Code of Ethics \url{https://neurips.cc/public/EthicsGuidelines}?
    \item[] Answer: \answerYes{} 
    \item[] Justification: 
    We have read NeurIPS Code of Ethics and carried out our research accordingly.
    \item[] Guidelines:
    \begin{itemize}
        \item The answer NA means that the authors have not reviewed the NeurIPS Code of Ethics.
        \item If the authors answer No, they should explain the special circumstances that require a deviation from the Code of Ethics.
        \item The authors should make sure to preserve anonymity (e.g., if there is a special consideration due to laws or regulations in their jurisdiction).
    \end{itemize}

\item {\bf Broader impacts}
    \item[] Question: Does the paper discuss both potential positive societal impacts and negative societal impacts of the work performed?
    \item[] Answer: \answerYes{} 
    \item[] Justification: 
    Discussed in \autoref{app:impact}.
    \item[] Guidelines:
    \begin{itemize}
        \item The answer NA means that there is no societal impact of the work performed.
        \item If the authors answer NA or No, they should explain why their work has no societal impact or why the paper does not address societal impact.
        \item Examples of negative societal impacts include potential malicious or unintended uses (e.g., disinformation, generating fake profiles, surveillance), fairness considerations (e.g., deployment of technologies that could make decisions that unfairly impact specific groups), privacy considerations, and security considerations.
        \item The conference expects that many papers will be foundational research and not tied to particular applications, let alone deployments. However, if there is a direct path to any negative applications, the authors should point it out. For example, it is legitimate to point out that an improvement in the quality of generative models could be used to generate deepfakes for disinformation. On the other hand, it is not needed to point out that a generic algorithm for optimizing neural networks could enable people to train models that generate Deepfakes faster.
        \item The authors should consider possible harms that could arise when the technology is being used as intended and functioning correctly, harms that could arise when the technology is being used as intended but gives incorrect results, and harms following from (intentional or unintentional) misuse of the technology.
        \item If there are negative societal impacts, the authors could also discuss possible mitigation strategies (e.g., gated release of models, providing defenses in addition to attacks, mechanisms for monitoring misuse, mechanisms to monitor how a system learns from feedback over time, improving the efficiency and accessibility of ML).
    \end{itemize}
    
\item {\bf Safeguards}
    \item[] Question: Does the paper describe safeguards that have been put in place for responsible release of data or models that have a high risk for misuse (e.g., pretrained language models, image generators, or scraped datasets)?
    \item[] Answer: \answerNA{} 
    \item[] Justification: NA. 
    \item[] Guidelines:
    \begin{itemize}
        \item The answer NA means that the paper poses no such risks.
        \item Released models that have a high risk for misuse or dual-use should be released with necessary safeguards to allow for controlled use of the model, for example by requiring that users adhere to usage guidelines or restrictions to access the model or implementing safety filters. 
        \item Datasets that have been scraped from the Internet could pose safety risks. The authors should describe how they avoided releasing unsafe images.
        \item We recognize that providing effective safeguards is challenging, and many papers do not require this, but we encourage authors to take this into account and make a best faith effort.
    \end{itemize}

\item {\bf Licenses for existing assets}
    \item[] Question: Are the creators or original owners of assets (e.g., code, data, models), used in the paper, properly credited and are the license and terms of use explicitly mentioned and properly respected?
    \item[] Answer: \answerYes{} 
    \item[] Justification: 
    PDE benchmarks acknowledged in the main paper. 
    \item[] Guidelines:
    \begin{itemize}
        \item The answer NA means that the paper does not use existing assets.
        \item The authors should cite the original paper that produced the code package or dataset.
        \item The authors should state which version of the asset is used and, if possible, include a URL.
        \item The name of the license (e.g., CC-BY 4.0) should be included for each asset.
        \item For scraped data from a particular source (e.g., website), the copyright and terms of service of that source should be provided.
        \item If assets are released, the license, copyright information, and terms of use in the package should be provided. For popular datasets, \url{paperswithcode.com/datasets} has curated licenses for some datasets. Their licensing guide can help determine the license of a dataset.
        \item For existing datasets that are re-packaged, both the original license and the license of the derived asset (if it has changed) should be provided.
        \item If this information is not available online, the authors are encouraged to reach out to the asset's creators.
    \end{itemize}

\item {\bf New assets}
    \item[] Question: Are new assets introduced in the paper well documented and is the documentation provided alongside the assets?
    \item[] Answer: \answerYes{} 
    \item[] Justification: 
    Code will be released at \url{https://github.com/csiro-funml/nots}.
    \item[] Guidelines:
    \begin{itemize}
        \item The answer NA means that the paper does not release new assets.
        \item Researchers should communicate the details of the dataset/code/model as part of their submissions via structured templates. This includes details about training, license, limitations, etc. 
        \item The paper should discuss whether and how consent was obtained from people whose asset is used.
        \item At submission time, remember to anonymize your assets (if applicable). You can either create an anonymized URL or include an anonymized zip file.
    \end{itemize}

\item {\bf Crowdsourcing and research with human subjects}
    \item[] Question: For crowdsourcing experiments and research with human subjects, does the paper include the full text of instructions given to participants and screenshots, if applicable, as well as details about compensation (if any)? 
    \item[] Answer: \answerNA{} 
    \item[] Justification: 
    NA.
    \item[] Guidelines:
    \begin{itemize}
        \item The answer NA means that the paper does not involve crowdsourcing nor research with human subjects.
        \item Including this information in the supplemental material is fine, but if the main contribution of the paper involves human subjects, then as much detail as possible should be included in the main paper. 
        \item According to the NeurIPS Code of Ethics, workers involved in data collection, curation, or other labor should be paid at least the minimum wage in the country of the data collector. 
    \end{itemize}

\item {\bf Institutional review board (IRB) approvals or equivalent for research with human subjects}
    \item[] Question: Does the paper describe potential risks incurred by study participants, whether such risks were disclosed to the subjects, and whether Institutional Review Board (IRB) approvals (or an equivalent approval/review based on the requirements of your country or institution) were obtained?
    \item[] Answer: \answerNA{} 
    \item[] Justification: 
    NA.
    \item[] Guidelines:
    \begin{itemize}
        \item The answer NA means that the paper does not involve crowdsourcing nor research with human subjects.
        \item Depending on the country in which research is conducted, IRB approval (or equivalent) may be required for any human subjects research. If you obtained IRB approval, you should clearly state this in the paper. 
        \item We recognize that the procedures for this may vary significantly between institutions and locations, and we expect authors to adhere to the NeurIPS Code of Ethics and the guidelines for their institution. 
        \item For initial submissions, do not include any information that would break anonymity (if applicable), such as the institution conducting the review.
    \end{itemize}

\item {\bf Declaration of LLM usage}
    \item[] Question: Does the paper describe the usage of LLMs if it is an important, original, or non-standard component of the core methods in this research? Note that if the LLM is used only for writing, editing, or formatting purposes and does not impact the core methodology, scientific rigorousness, or originality of the research, declaration is not required.
    \item[] Answer: \answerNA{} 
    \item[] Justification: 
    NA.
    \item[] Guidelines:
    \begin{itemize}
        \item The answer NA means that the core method development in this research does not involve LLMs as any important, original, or non-standard components.
        \item Please refer to our LLM policy (\url{https://neurips.cc/Conferences/2025/LLM}) for what should or should not be described.
    \end{itemize}

\end{enumerate}
\newpage
\appendix
\renewcommand{\sectionautorefname}{App.}
\renewcommand{\subsectionautorefname}{App.}

\section*{Appendix}
We now present detailed theoretical background, proofs, experiment settings, and additional results that complement the main paper. \autoref{app:background} reviews essential background on the infinite-width limit of neural networks \citep{Lee2019} and how they relate to Gaussian processes \citep{Rasmussen2006}. We discuss the distinction and applicability of the two main kernel-based frameworks suitable for this type of analysis, namely, the neural tangent kernel (NTK) by \citet{Jacot2018} and the conjugate kernel, also known as the neural network Gaussian process (NNGP) kernel \citep{Neal1996, Lee2018}, which was the main tool for our derivations. \autoref{app:infinitewidth} formulates Fourier neural operators \citep{Li2021fno} under the mathematical abstraction that allowed us to derive the operator-valued kernel for neural operators. The proofs of the main theoretical results then appear in \autoref{app:theory}, including the construction and properties of the operator-valued kernel and the correspondence between trained neural operators and their GP limits. \autoref{app:experiments} describes the PDE benchmarks considered, namely Darcy flow and shallow water equations, alongside the respective objective functionals for optimization tasks. Experiment details, hyperparameter settings, and baseline implementation details are provided in \autoref{app:settings}. \autoref{app:extra} presents results on an experiment with a single-hidden-layer neural operator validating our theoretical results.
Lastly, we discuss limitations and potential broader impact in sections \ref{app:limitations} and \ref{app:impact}, respectively.

\section{Additional background}
\label{app:background}

In this section, we discuss the main differences between the neural tangent kernel \citep{Jacot2018} and the conjugate kernel, also known as the neural network Gaussian process (NNGP) kernel \citep{Lee2019}. Both kernels are used to approximate the behavior of neural networks, but they differ in how they use Gaussian processes to describe the network's behavior.

\subsection{Conjugate kernel (NNGP)}
The conjugate kernel has long been studied in the neural networks literature, describing the correspondence neural networks with randomized parameters and their limiting distribution as the network width approaches infinity \citep{Neal1996, Daniely2017, Lee2018, Matthews2018, Hu2021}. \citet{Neal1996} first showed the correspondence between an infinitely wide single-hidden-layer network and a Gaussian process by applying the central limit theorem. More recent works \citep{Daniely2017, Lee2018, Matthews2018} later showed that the same reasoning can be extended to neural networks with multiple hidden layers. The NNGP kernel is particularly useful for Bayesian inference as it allows us to define GP priors for neural networks and analyze how they change when conditioned on data, providing us with closed-form expressions for an exact GP posterior in the infinite-width limit \citep{Lee2018}.

Define an $\nlayers$-layer neural network $\model(\cdot, \parameters): \domain \to \R$ with $\model(\location; \parameters) := \model_\nlayers(\location;\parameters)$ via the recursion:
\begin{equation}
    \begin{split}
        \model_0(\location; \parameters) &:= \location\\
        \model_{\layer}(\location; \parameters) &:= \actfun_\layer(\Weights_\layer\model_{\layer-1}(\location; \parameters) + \vec\nnbias_\layer) \,, \quad \layer \in \{1,\dots, \nlayers\},
        \label{eq:nn}
    \end{split}
\end{equation}
where $\location\in\domain$ represents an arbitrary input on a finite-dimensional domain $\domain$, $\Weights_\layer \in \R^{\paramdim_{\layer}\times\paramdim_{\layer-1}}$ denotes a layer's weights matrix, $\paramdim_\layer$ is the width of the $\layer$th layer, $\vec\nnbias_\layer \in \R^{\paramdim_\layer}$ is a bias vector, $\actfun_\layer:\R\to\R$ denotes the layer's activation function, which is applied elementwise on vector-valued inputs, and $\parameters := \operatorname{vec}(\{\Weights_\layer, \nnbias_\layer\}_{\layer=1}^\nlayers)$ collects all the network parameters into a vector. Assume $[\Weights_\layer]_{i,j} \sim \normal\left(0,\frac{1}{\paramdim_{\layer-1}}\right)$ and $[\vec\nnbias_\layer]_i\sim\normal(0,1)$, for $i \in \{1,\dots, \paramdim_{\layer}\}$, $j\in\{1,\dots,\paramdim_{\layer-1}\}$ and $\layer\in\{1,\dots,\nlayers\}$, and let $\paramdim := \min\{\paramdim_1,\dots,\paramdim_\nlayers\}$. The NNGP kernel then corresponds to the infinite-width limit of the network outputs covariance function \citep{Lee2018} as:
\begin{equation}
    \kernel_\nngp(\location,\location') := \lim_{\paramdim\to\infty} \expectation[\model(\location; \parameters)\model(\location'; \parameters)], \quad \location,\location'\in\domain \,,
    \label{eq:nngp}
\end{equation}
where the expectation is taken under the parameters distribution. By an application of the central limit theorem, it can be shown \citep{Neal1996, Lee2018} that the neural network converges in distribution to a Gaussian process with the kernel defined above, i.e.:
\begin{equation}
    \model_\parameters \xrightarrow{d} \model \sim \gp(0, \kernel_\nngp)\,,
    \label{eq:nnlimit}
\end{equation}
where $\xrightarrow{d}$ denotes convergence in distribution as $\paramdim\to\infty$. In other words, the randomly initialized network follows a GP prior in the infinite-width limit. Moreover, it follows that, when conditioned on data $\dataset_\nObs := \{\location_i, \observation_i\}_{i=1}^\nObs$, assuming $\observation_i = \model(\location_i) + \obsnoise_i$ and $\obsnoise_i \sim \normal(0, \sigma_\obsnoise^2)$, a Bayesian neural network is distributed according to a GP posterior in the infinite-width limit as:
\begin{align}
    \model|\dataset_\nObs &\sim \gp(\gpmean_\nObs, \kernel_\nObs) \label{eq:gp-posterior}\\
    \gpmean_\nObs(\location) := \expectation[\model(\location) \mid \dataset_\nObs] &= \vkernel_\nObs(\location)^\transpose (\Kernel_\nObs + \sigma_\obsnoise^2\eye)^{-1}\observations_\nObs\\
    \kernel_\nObs(\location,\location') := \covariance[\model(\location), \model(\location') \mid \dataset_\nObs] &= \kernel(\location, \location') - \vkernel_\nObs(\location)^\transpose(\Kernel_\nObs + \sigma_\obsnoise^2\eye)^{-1}\vkernel_\nObs(\location'),
\end{align}
for any $\location, \location'\in \domain$, where $\Kernel_\nObs := [\kernel(\location_i, \location_j)]_{i, j = 1}^\nObs \in \R^{\nObs\times\nObs}$, $\vkernel_\nObs(\location) := [\kernel(\location_i, \location)]_{i=1}^\nObs \in \R^\nObs$, $\observations_\nObs := [\observation_i]_{i=1}^\nObs$, and we set $\kernel := \kernel_\nngp$ to avoid notation clutter. Hence, the NNGP kernel allows us to compute exact GP posteriors for neural network models. However, we emphasize that the conjugate kernel should not be confused with the neural tangent kernel \citep{Jacot2018}, which corresponds to the infinite-width limit of $\expectation[\nabla_\parameters\model(\location;\parameters)\cdot\nabla_\parameters\model(\location';\parameters)]$, instead.

\subsection{Neural tangent kernel (NTK)}
The NTK approximates the behavior of a neural network during training via gradient descent by considering the gradients of the network with respect to its parameters \citep{Jacot2018}. Consider an $\nlayers$-layer feedforward neural network $\model_\parameters: \domain \to \R$ as defined in \autoref{eq:nn}. In its original formulation, \citet{Jacot2018} applied a scaling factor of $\frac{1}{\sqrt{\paramdim}}$ to the output of each layer to ensure asymptotic convergence in the limit $\paramdim\to\infty$ of the network trained via gradient descent. However, later works showed that standard network parameterizations (without explicit output scaling) also converge to the same limit as long as a LeCun or Kaiming/He type of initialization scheme is applied to the parameters with appropriate scaling of the learning rates \citep{Lee2019, Liu2020}, which ensure bounded variance in the infinite-width limit. The NTK describes the limit:
\begin{equation}
    \kernel_\ntk(\location, \location') = \lim_{\paramdim \to \infty} \expectation[\nabla_\parameters \model_\parameters(\location) \cdot \nabla_\parameters \model_\parameters(\location')]\,,
    \label{eq:ntk}
\end{equation}
for any $\location,\location' \in \domain$, where the expectation is taken under the parameters initialization distribution. Under mild assumptions, the trained network's output distribution converges to a Gaussian process described by the NTK \citep{Jacot2018, Lee2018}. Although originally derived for the unregularized case, applying L2 regularization to the parameters norm yields a GP posterior with a term that can account for observation noise \citep{Ordonez2025observation}. Namely, consider the following loss function:
\begin{equation}
    \loss_\nObs(\parameters) := \sum_{i=1}^{\nObs} (\observation_i - \model_\parameters(\location_i))_2^2 + \regFactor \norm{\parameters - \parameters_0}_2^2\,,
\end{equation}
where $\parameters_0$ denotes the initial parameters. As the network width grows larger, the NTK tells us that the network behaves like a linear model \citep{Jacot2018, Liu2020} as:
\begin{equation}
    \model(\location; \parameters) \approx \model(\location; \parameters_0) + \nabla_{\parameters} \model(\location; \parameters)\big|_{\parameters:=\parameters_0} \cdot (\parameters - \parameters_0)\,, \quad \location\in\domain\,.
\end{equation}
The approximation becomes exact in the infinite width limit within any bounded neighborhood $\ball_\radius(\parameters_0) := \{\parameters \mid \norm{\parameters - \parameters_0} \leq \radius \}$ of arbitrary radius $0 < \radius < \infty$ around $\parameters_0$, as the second-order error term vanishes \citep{Liu2020}. The latter also means that $\nabla_{\parameters} \model(\cdot; \parameters)$ converges to fixed feature map $\feature:\domain\to\Hspace_0$, where $\Hspace_0$ is the Hilbert space spanned by the limiting gradient vectors. With this observation, our loss function can be rewritten as:
\begin{equation}
    \loss_\nObs(\parameters) \approx \sum_{i=1}^\nObs \left(\observation_i - \model(\location_i; \parameters_0) - \nabla_{\parameters} \model(\location_i; \parameters)\big|_{\parameters:=\parameters_0} \cdot (\parameters - \parameters_0) \right)^2 + \regFactor \norm{\parameters - \parameters_0}_2^2\,.
\end{equation}
The minimizer of the approximate loss can be derived in closed form. Applying the NTK then yields the infinite-width model:
\begin{equation}
    \model_\nObs(\location) = \model(\location) + \vkernel_\nObs^\ntk(\location)^\transpose(\Kernel_\nObs^\ntk + \regFactor\eye)^{-1}(\observations_\nObs - \vec\model_\nObs)\,,
\end{equation}
where $\model \sim \gp(0, \kernel_\nngp)$ denotes the network at its random initialization, as defined above, $\vkernel_\nObs^\ntk(\location) := [\kernel_\ntk(\location_i, \location)]_{i=1}^\nObs \in \R^\nObs$, $\Kernel_\nObs^\ntk := [\kernel_\ntk(\location_i, \location_j)]_{i, j = 1}^\nObs \in \R^{\nObs\times\nObs}$, and $\vec\model_\nObs := [\model(\location_i)]_{i=1}^\nObs \in \R^\nObs$. Now applying the GP limit to the randomly initialized network $\model$ \citep{Lee2019, Ordonez2025observation}, we have that:
\begin{align}
    \model_\nObs &\sim \gp(\hat\gpmean_\nObs, \hat\kernel_\nObs) \label{eq:ntk-posterior}\\
    \hat\gpmean_\nObs(\location) &= \vkernel_\nObs^\ntk(\location)^\transpose(\Kernel_\nObs^\ntk + \regFactor\eye)^{-1}\observations_\nObs\\
    \begin{split}
        \hat\kernel_\nObs(\location,\location') &= \kernel(\location, \location')
        + \vkernel_\nObs^\ntk(\location)^\transpose (\Kernel_\nObs^\ntk + \regFactor\eye)^{-1} \mat\Kernel_\nObs (\Kernel_\nObs^\ntk + \regFactor\eye)^{-1}\vkernel_\nObs^\ntk(\location') \\
        &\quad - \vkernel_\nObs^\ntk(\location)^\transpose(\Kernel_\nObs^\ntk + \regFactor\eye)^{-1} \vkernel_\nObs(\location') - \vkernel_\nObs(\location)^\transpose(\Kernel_\nObs^\ntk + \regFactor\eye)^{-1}\vkernel_\nObs^\ntk(\location'),
    \end{split}
\end{align}
where we again set $\kernel := \kernel_\nngp$ to avoid clutter. However, note that such GP model does not generally correspond to a Bayesian posterior. An exception is where only the last linear layer is trained, while the rest are kept fixed at their random initialization; in which case case, the GP described by the NTK and the exact GP posterior according to the NNGP kernel match in the unregularized setting \citep{Lee2019}.

\subsection{Application to Thompson sampling}
For our purpose, it is important to have a Bayesian posterior in order to apply Gaussian process Thompson sampling (GP-TS) \citep{Takeno2024} for the regret bounds in \autoref{thm:regretbound}.
Therefore, we are constrained by existing theories connecting neural networks to Gaussian processes to assume training only the last layer of neural networks of infinite width,
which gives a Bayesian posterior of the NNGP after training. In addition, we had to consider the case of a single hidden layer neural operator, as the usual recursive step applied to derive the infinite-width limit would require an intermediate (infinite-dimensional) function space in our case, making the extension to the multi-layer case not trivial due to the usual finite-dimensional assumptions \citep{Liu2020}.
Nonetheless, the NOTS algorithm suggested by our theory has demonstrated competitive performance in our experiments even in more relaxed settings with a multi-layer model.
Future theoretical developments in Bayesian analysis of neural networks may eventually permit the convergence analysis of the more relaxed settings in our experiments. In any case, we present an experiment with a wide single-hidden-layer model with training only on the last layer in \autoref{app:extra}.

\section{Fourier neural operators under the abstract representation}
\label{app:infinitewidth}
Recalling the definition in the main paper, we consider a single hidden layer neural operator. Let $\paramdim\in\N$ represent the layer's width, $\nnkop:\infspace\to\cspace(\outdomain, \R^{\dimension_\nnkernel})$ denote a (fixed) continuous operator, and $\nnbias_0:\outdomain\to\R^{\dimension_\nnbias}$ denote a (fixed) continuous function. For simplicity, we assume scalar outputs with $\outdim=1$. We consider models of the form:
\begin{equation}
    \nnop_\parameters(\infun)(\outlocation) = \weights_o^\transpose\actfun\left( \Weights_\nnkernel\nnkop(\infun)(\outlocation) + \Weights_\outfun\infun(\nnproj_0(\outlocation)) + \Weights_\nnbias\nnbias_0(\outlocation) \right)\,, \quad \outlocation\in\outdomain\,,
    \label{eq:app-nnop-abstraction}
\end{equation}
where $\parameters := (\weights_o, \Weights_\nnkernel, \Weights_\outfun, \Weights_\nnbias) \in \R^\paramdim\times\R^{\paramdim\times\dimension_\nnkernel}\times\R^{\paramdim\times\indim}\times\R^{\paramdim\times\dimension_\nnbias} =: \paramspace$ represents parameters.

\paragraph{Fourier neural operators.} As an example, we show how the formulation above applies to the Fourier neural operator (FNO) architecture \citep{Li2021fno}. For simplicity, assume that $\domain$ is the $\dimension$-dimensional periodic torus, i.e., $\domain = [0, 2\pi)^\dimension$, and $\outdomain=\domain$. Then any square-integrable function $\infun: \domain \to \C^\indim$ can be expressed as a Fourier series:
\begin{equation}
    \infun(\location) = \sum_{\freq \in \Z^\dimension} \hat\infun(\freq) e^{\imunit\inner{\freq,\location}}, \quad \forall \location\in\domain\,,
\end{equation}
where $\imunit:=\sqrt{-1} \in \C$ denotes the imaginary unit, and $\hat\infun(\freq)$ are coefficients given by the function's Fourier transform $\Fourier: \lpspace{2}(\domain, \C^\indim)\to\lpspace{2}(\Z^\dimension, \C^\indim)$ as:
\begin{equation}
    \hat\infun(\freq) := (\Fourier\infun)(\freq) = \frac{1}{(2\pi)^\dimension}\int_\domain \infun(\location) e^{-\imunit\inner{\freq,\location}}\diff\location\,, \quad \freq\in\Z^\dimension\,.
\end{equation}
For a translation-invariant kernel $\nnkernel(\location,\location') = \nnkernel(\location-\location')$, applying the convolution theorem, the integral operator can be expressed as:
\begin{equation}
    \begin{split}
        \int_{\domain}
            \nnkernel(
                \cdot, \location
                )
            \infun(\location)\diff\location
            &= \nnkernel * \infun\\
            &= \Fourier^{-1}(\Fourier(\nnkernel) \cdot \Fourier(\infun))\\
            &= \sum_{\freq \in \Z^\dimension} \widehat\nnkernel(\freq) \hat\infun(\freq) e^{\imunit\inner{\freq,\cdot}}
    \end{split}
\end{equation}
In practice, function observations are only available at a discrete set of points and the Fourier series is truncated at a maximum frequency $\freq_{\max}\in\Z^\dimension$, which allows one to efficiently compute it via the fast Fourier transform (FFT). Considering these facts, FNOs approximate the integral as \citep{Li2021fno}:
\begin{equation}
    \begin{split}
        \int_{\domain}
            \nnkernel(
                \location, \location'
                )
            \infun(\location')\diff\location'
            &\approx \sum_{n=1}^N \widehat\nnkernel(\freq_n) \hat\infun(\freq_n) e^{\imunit\inner{\freq_n,\location}}, \quad \location\in\outdomain\,,
    \end{split}
\end{equation}
where the $N$ values of $\freq_n$ range from 0 to $\freq_{\max}$ in all $\dimension$ coordinates. Finally, defining $\nnkop$ as:
\begin{equation}
    \begin{split}
        \nnkop: \cspace(\domain,\C^\indim) &\to \cspace(\domain,\C^{N\indim})\\
        \infun &\mapsto
        \begin{bmatrix}
            (\Fourier\infun)(\freq_1) e^{\imunit\inner{\freq_1,\cdot}}\\
            \vdots\\
            (\Fourier\infun)(\freq_N) e^{\imunit\inner{\freq_N,\cdot}}
        \end{bmatrix}\,,
    \end{split}
\end{equation}
and letting $\Weights_\nnkernel = [\widehat\nnkernel(\freq_1),\dots, \widehat\nnkernel(\freq_N)]$, we recover  \autoref{eq:app-nnop-abstraction} for FNOs in the complex-valued case.

For real-valued functions, to ensure that the result is again real-valued, a symmetry condition is imposed on $\widehat\nnkernel$, so that its values for negative frequencies are the conjugate transpose of the corresponding values for positive frequencies. However, we can still represent it via a single matrix of weights, which is simply conjugate transposed for the negative frequencies. Lastly, note that complex numbers can be represented as tuples of real numbers.

\section{Theoretical Analysis}
\label{app:theory}
In this section, we provide the proofs of the theoretical results presented in the main paper.

\subsection{Auxiliary results}

\begin{definition}[Multi-Layer Fully-Connected Neural Network]
    \label{def:fcnn}
    A multi-layer fully-connected neural network with \( L \) hidden layers, input dimension \( \dimension_0 \), output dimension \( \dimension_{L+1} \), and hidden layer widths \( \dimension_1, \ldots, \dimension_L \), is defined recursively as follows. For input \( \location \in \domain \), the pre-activations and activations at layer \( \layer = 1, \ldots, L+1 \) are:
    \begin{align}
        \nninput^{(1)}(\location) &= \Weights^{(0)} \location + \vec\nnbias^{(0)}\\
        \nninput^{(\layer)}(\location) &= \Weights^{(\layer-1)} \actfun(\nninput^{(\layer-1)}(\location)) + \vec\nnbias^{(\layer-1)}, \quad \layer=2,\dots,L,\\
        \nninput^{(L+1)}(\location) &= \Weights^{(L)} \actfun(\nninput^{(L)}(\location)),
    \end{align}
    where \( \Weights^{(\layer)} \in \R^{\dimension_{\layer+1} \times \dimension_\layer} \) are weight matrices, \( \vec\nnbias^{(\layer)} \in \R^{\dimension_{\layer+1}} \) are bias vectors, \( \actfun:\R \to \R \) is a coordinate-wise non-linearity, and the network output is \( f(\location) = \nninput^{(L+1)}(\location) \). The weights are initialized as \( W_{ij}^{(\layer)} = \left(\frac{\anyconstant_W}{\dimension_\layer}\right)^{1/2} \widehat{W}_{ij}^{(\layer)} \), where \( \widehat{W}_{ij}^{(\layer)} \sim \mu \) with mean 0, variance 1, and finite higher moments, and biases as \( b_i^{(\layer)} \sim \mathcal{N}(0, \anyconstant_b) \), given fixed constants $\anyconstant_W > 0$ and $\anyconstant_b \geq 0$.
\end{definition}

\begin{lemma}[{Infinite-width limit \citep{Hanin2023nngp}}]
    \label{thr:infinite-width-nn}
    Consider a feedforward fully connected neural network as in \autoref{def:fcnn} with non-linearity \( \actfun: \R\to\R \) that is absolutely continuous with polynomially bounded derivative. Fix the input dimension \( \dimension_0 \), the output dimension \( \dimension_{L+1} \), the number of layers \( L \), and a compact set \( \domain \subset \R^{\dimension_0} \). As hidden layer widths \( \dimension_1, \ldots, \dimension_L \to \infty \), the random field \( x \mapsto f(x) \) converges weakly in \( \cspace(\domain, \R^{\dimension_{L+1}}) \) to a centered Gaussian process with covariance \( \mat{\kernel}^{(L+1)}:\domain\times\domain\to\R^{\dimension_{\nlayers+1} \times \dimension_{\nlayers+1}} \) defined recursively by:
    \begin{equation}
        \mat{\kernel}^{(\layer+1)}(x, x') = \anyconstant_b \eye + \anyconstant_W \, \expectation_{(\nninput,\nninput')} \left[ \actfun(\nninput) \otimes \actfun(\nninput') \right],
    \end{equation}
    where \( (\nninput, \nninput') \sim \mathcal{N} \left( 0, \begin{bmatrix}
        \mat{\kernel}^{(\layer)}(x, x) & \mat{\kernel}^{(\layer)}(x, x') \\
        \mat{\kernel}^{(\layer)}(x, x') & \mat{\kernel}^{(\layer)}(x', x')
    \end{bmatrix} \right) \) for \( \layer \geq 2 \), with the initial condition for \( \layer=1 \) determined by the first-layer weights and biases.
\end{lemma}

\begin{lemma}[{Thm. 3.1 in \citet{Takeno2024}}]
    \label{thr:gp-ts-regret}
    Let $\objective \sim \gp(0, \kernel)$, where $\kernel: \domain\times\domain \to \R$ is a positive-definite kernel on a finite $\domain$. Then the Bayesian cumulative regret of GP-TS is such that:
    \begin{equation*}
        \bcr_\nIterations \in \bigo(\sqrt{\nIterations\mig_\nIterations})\,,
    \end{equation*}
    where $\mig_\nIterations$ denotes the maximum information gain after $\nIterations$ iterations with the GP model.
\end{lemma}


\subsection{Infinite-width neural operator kernel}
\label{app:infinite-width-limit}

\begin{assumption}
    \label{a:activation}
    The activation function $\actfun:\R\to\R$ is absolutely continuous with derivative bounded almost everywhere.
\end{assumption}

\begin{lemma}[Continuity of limiting GP]
    \label{thr:continuity}
    Let $\nnop_\parameters: \infspace \to \cspace(\outdomain)$ be a neural operator with a single hidden layer, as defined as in \autoref{eq:app-nnop-abstraction}. Assume $\weights_o \sim \normal(\vec 0, \sigma_\parameters^2\eye)$, for $\sigma_\parameters^2 > 0$ such that $\sigma_\parameters^2  \propto\frac{1}{\paramdim}$, and let the remaining parameters have their entries be sampled from a fixed normal distribution. Then, as $\paramdim\to\infty$, the neural operator converges in distribution to a zero-mean Gaussian process with continuous realizations $\nnop:\infspace'\to\cspace(\outdomain)$ on every compact subset $\infspace' \subset \infspace$.
\end{lemma}
\begin{proof}
As shown in \autoref{sec:nnop-limit}, when evaluated at a fixed point $\outlocation\in\outdomain$, a neural operator with a single hidden layer can be seen as:
\begin{equation}
    \nnop_\parameters(\infun)(\outlocation) = \model_\parameters(\vec{\psi}(\infun, \outlocation)), \quad \infun\in\infspace\,,
\end{equation}
where $\vec{\psi}(\infun, \outlocation) := \nninput_\outlocation(\infun)$ is a fixed map $\vec{\psi}:\infspace\times\outdomain\to\nndomain$, with $\nndomain = \R^{\dimension_\nnkernel + \indim + \dimension_\nnbias}$, and $\model_\parameters$ is a conventional feedforward neural network, as defined in \autoref{def:fcnn}. By \autoref{a:activation} and \autoref{thr:infinite-width-nn}, it follows that, as $\paramdim\to\infty$, $\model_\parameters$ converges in distribution to a Gaussian process $\model \sim \gp(0, \kernel_\model)$ with continuous sample paths, i.e., $\prob{\model \in \cspace(\nndomain')}=1$ on every compact $\nndomain'\subset \nndomain$. The continuity of $\vec{\psi}:\infspace\times\outdomain\to\nndomain$ then implies that $\gpfunction := \model \circ \vec{\psi}$ is a zero-mean GP whose sample paths lie almost surely in $\cspace(\infspace'\times\outdomain)$, for a compact $\infspace'\subset\infspace$, as $\outdomain$ is already assumed compact. Therefore, for each $\infun\in\infspace$, we have $\prob{\gpfunction(\infun, \cdot) \in \cspace(\outdomain)} = 1$, so that $\nnop(\infun) := \gpfunction(\infun, \cdot)$ defines an almost surely continuous operator $\nnop:\infspace'\to\cspace(\outdomain)$ on compact $\infspace'\subset\infspace$. The verification that $\nnop$ is a vector-valued GP trivially follows.
\end{proof}

\ckresult*
\begin{proof}[Proof of Proposition \ref{thr:kernel}]
    We start by noting that any continuous function $\outfun \in \cspace(\outdomain)$ is automatically included in $\lpspace{2}(\lpmeasure)$, since $\norm{\outfun}_{\lpspace{2}(\lpmeasure)}^2 = \int_\outdomain \outfun^2(\outlocation) \diff\lpmeasure(\outlocation) \leq \lpmeasure(\outdomain) \norm{\outfun}_\infty^2 < \infty$. Hence, any operator mapping into $\cspace(\outdomain)$ also maps into $\lpspace{2}(\lpmeasure)$ by inclusion.

    Applying \autoref{thr:continuity}, it follows that $\nnop_\parameters \overset{d}{\to} \nnop$, where $\nnop$ is a zero-mean GP, as $\paramdim \to \infty$. Now, given any $\outfun \in \outfspace$, $\infun, \infun'\in\infspace$ and $\outlocation\in\outdomain$, we have that:
    \begin{equation}
        \begin{split}
            (\expectation[\nnop(\infun) \otimes \nnop(\infun')] \outfun)(\outlocation) &= \expectation[\nnop(\infun)\inner{\nnop(\infun'), \outfun}]\\
            &= \left(\expectation\left[\gpfunction(\infun, \cdot) \int_\outdomain \gpfunction(\infun', \outlocation') \outfun(\outlocation') \diff\lpmeasure(\outlocation') \right]\right)(\outlocation)\\
            &= \expectation\left[ \int_\outdomain \gpfunction(\infun, \outlocation)\gpfunction(\infun', \outlocation') \outfun(\outlocation') \diff\lpmeasure(\outlocation') \right]\\
            &= \int_\outdomain \expectation[\gpfunction(\infun, \outlocation)\gpfunction(\infun', \outlocation')] \outfun(\outlocation') \diff\lpmeasure(\outlocation')\\
            &= \int_\outdomain \kernel_\nnop(\infun, \outlocation, \infun', \outlocation') \outfun(\outlocation') \diff\lpmeasure(\outlocation')\,,
        \end{split}
    \end{equation}
    where we applied the linearity of expectations and the correspondence between $\gpfunction: \infspace\times\outdomain\to\R$ and the limiting operator $\nnop:\infspace\to\outfspace$. As the choice of elements was arbitrary, it follows that the above defines an operator-valued kernel $\opkernel_\nnop$. Linearity follows from the expectations. Given any $\infun\in\infspace$, as a positive-semidefinite operator, the operator norm of $\opkernel_\nnop(\infun,\infun)$ is bounded by its trace, such that:
    \begin{equation}
        \norm{\opkernel_\nnop(a,a)}^2 \leq \tr(\opkernel_\nnop(\infun,\infun)) = \expectation[\norm{\nnop(\infun)}_\outfspace^2] = \expectation\left[ \int_\outdomain \gpfunction^2(\infun, \outlocation) \diff\lpmeasure(\outlocation) \right] < \lpmeasure(\outdomain)\expectation[\norm{\gpfunction(a,\cdot)}_\infty^2] \,,
    \end{equation}
    and the last expectation is finite, since $\gpfunction$ is almost surely continuous. Hence, $\opkernel_\nnop(\infun,\infun) \in \lopspace(\outfspace)$.
\end{proof}

\subsection{Regret bound}

\regretbound*
\begin{proof}[{Proof of \autoref{thm:regretbound}}]
    By linearity, it follows that $\objective \circ \nnop_* \sim \gp(0, \objective^\transpose \opkernel \objective)$ for any fixed bounded linear functional $\objective:\outfspace\to\R$. Hence, $\objective \circ \nnop_*$ is equal in distribution to a scalar-valued GP $\model \sim \gp(0, \kernel_\objective)$, where $\kernel_\objective : \infspace \times \infspace \to \R$ is given by:
    \begin{equation*}
        \kernel_\objective(\infun,\infun') = \objective(\opkernel(\infun,\infun')\objective), \quad \infun, \infun'\in\infspace,
    \end{equation*}
    where we implicitly identify the functional $\objective$ with a unique corresponding vector in $\outfspace$, also denoted by $\objective$, by the Riesz representation theorem to apply the operator $\opkernel(\infun,\infun') \in \lopspace(\outfspace)$ to $\objective$. By \autoref{thr:gp-ts-regret}, standard GP-TS on an objective $\model \sim \gp(0, \kernel_\objective)$, a finite domain $\Sspace\subset\infspace$ 
    will have Bayesian cumulative regret $\Regret_\nIterations \in \bigo(\sqrt{\nIterations\mig_{\objective,\nIterations}})$. Note that $\mig_{\objective,\nIterations}$ corresponds to the maximum information gain after $\nIterations$ observations, where each observation is a vector $\observation_\iterIdx\in\obsspace\subseteq\R^\obsdim$, not a scalar as it would be usually assumed in GP-TS. However, the proof of \autoref{thr:gp-ts-regret} in \citet[Thm. 3.1]{Takeno2024} does not depend on the particular form of the posterior mean $\expectation[\model(\infun) \mid \dataset_\iterIdx]$ or variance $\variance[\model(\infun) \mid \dataset_\iterIdx]$, as long as the posterior remains a GP, which still holds. Lastly, the restriction to $\objective = \tilde\objective\circ\obsop$ ensures that the null space of $\obsop$ and $\objective$ coincide, so that observations allow us to learn $\objective$.
\end{proof}

\begin{remark}
    Despite the result above assuming that $\objective$ is only a function of $\nnop(\infun)$, there is a straightforward extension to functionals of the form $\objective: \outfspace\times\infspace\to\R$, as considered in our experiments. We simply need to replace $\nnop:\infspace\to\outfspace$ with the operator $\nnop':\infun\mapsto(\nnop(\infun), \infun)$ by a concatenation with an identity map $\infun\mapsto\infun$, which is deterministic. A similar result then follows after minor adjustments.
\end{remark}

\begin{remark}
    For finite domains, we have that the maximum information gain of GP-TS is at most of logarithmic growth, i.e., $\mig_{\nIterations} \in \bigo(\log\nIterations)$, making the regret bound in \autoref{thm:regretbound} sublinear, regardless of the choice of operator-valued kernel $\opkernel$ and linear functional $\tilde\objective$. Indeed, the kernel matrix $\Kernel_{\nIterations}$ can have at most $\card{\Sspace} < \infty$ non-zero eigenvalues, $\eigval_1 \geq \dots \geq \eigval_\nIterations$, and the maximum eigenvalue $\eigval_1$ is bounded by the trace $\eigval_1 \leq \tr(\Kernel_\nIterations)$, which is at most $\bigo(\nIterations)$. Therefore, the log-determinant is such that: 
    \begin{equation*}
        \begin{split}
            \log\det(\eye + \regFactor^{-1}\Kernel_\nIterations)
            &= \sum_{\iterIdx=1}^\nIterations \log(1 + \regFactor^{-1}\eigval_\iterIdx)\\
            &= \sum_{\iterIdx=1}^{\card{\Sspace}} \log(1 + \regFactor^{-1}\eigval_\iterIdx)\\
            &\leq \card{\Sspace} \log(1 + \regFactor^{-1}\eigval_1)\\
            &\leq \card{\Sspace} \log(1 + \anyconstant\nIterations),
        \end{split}
    \end{equation*}
    for some $\anyconstant > 0$. As $\mig_\nIterations$ is the maximum of $\frac{1}{2}\log\det(\eye + \regFactor^{-1}\Kernel_\nIterations)$, we have $\mig_{\nIterations} \in \bigo(\log\nIterations)$.
\end{remark}

\subsection{Approximate posterior sampling via gradient descent}
\label{sec:sampling-by-gd}
We briefly review the equivalence between posterior sampling and gradient descent when training only the last (or readout) layer of a neural network under a (regularized) least-squares loss and LeCun (or Kaiming He) initialization in the presence of observation noise. We will mainly combine major results from the NTK and NNGP literature \citep{Lee2019, Liu2020, Ordonez2025observation} into the setting of our paper. When only the last layer is trained, the feature maps of the NTK and the NNGP coincide \citep[App. D]{Lee2019}, so that we can follow an NTK type of analysis of how the loss function relates to the network's parameters, while the distribution of the trained network is determined by the NNGP kernel. For simplicity, we focus on the case of a standard, fully connected, scalar-valued neural network, noticing that this analysis is readily extensible to the neural operator case by the techniques we use for our main results.

\paragraph{Random feature model.} When training only the last layer of a neural network, we have the following model at initialization:
\begin{equation}
    \model_0(\location) = \weights_0^\transpose \ffeature(\location),
\end{equation}
where we assume $\weights_0 \sim \normal(\vec 0, \frac{1}{\paramdim}\eye)$ for the initial weights of the readout layer, with $\paramdim$ representing the network width, and given $\location\in\domain$, $\ffeature(\location) \in \R^\paramdim$ represents the output of the last hidden layer of the neural network, which consists of a \emph{random feature} map $\ffeature:\domain\to\R^\paramdim$ under the initialization scheme. Observe that the NNGP kernel is given by:
\begin{equation}
    \kernel_\nngp(\location,\location') := \lim_{\paramdim\to\infty} \expectation[\model_0(\location)\model_0(\location')] = \lim_{\paramdim\to\infty} \frac{1}{\paramdim} \expectation[\ffeature(\location)^\transpose\ffeature(\location')]\,,
\end{equation}
for any $\location,\location'\in\domain$. Note that this is the same limit we obtain if $\weights_0 \sim \normal(\vec 0, \eye)$ and $\ffeature(\location)$ is scaled by $\frac{1}{\sqrt{\paramdim}}$, as in the NTK parameterization \citep{Jacot2018}. Hence, to simplify our derivations, we will adopt the latter in the remainder of this subsection.

\paragraph{Regularized least-squares estimator.} Given $\nObs$ data points $\dataset_\nObs := \{\location_i, \observation_i\}_{i=1}^\nObs \subset \domain\times \R$, we consider the following regularized least-squares loss:
\begin{equation}
    \loss_\nObs(\weights) := \frac{1}{2} \sum_{i=1}^\nObs (\weights^\transpose\ffeature(\location_i) - \observation_i)^2 + \frac{\regFactor}{2} \norm{\weights - \weights_0}^2 = \frac{1}{2}\norm{\features^\transpose\weights - \observations}^2 + \frac{\regFactor}{2} \norm{\weights - \weights_0}^2\,,
    \label{eq:loss-theory}
\end{equation}
where $\features := [\ffeature(\location)_1, \dots, \ffeature(\location_\nObs)] \in \R^{\paramdim\times\nObs}$, $\observations := [\observation_1, \dots, \observation_\nObs]^\transpose \in \R^\nObs$, $\weights_0 \sim \normal(\vec 0, \eye)$, and $\regFactor > 0$ is a regularization factor. We note that, in practice, due to the small initialization variance of order $\frac{1}{\paramdim}$, the initial weights $\weights_0$ will be elementwise very close to zero, especially for large widths $\paramdim$. Therefore, we omit $\weights_0$ from the regularizer in \autoref{eq:training}, as their practical effect is limited, and a simple L2 regularizer is typically efficiently implemented as a weight decay term in optimization algorithms found within modern deep learning frameworks, such as PyTorch \citep{Ansel2024pytorch}.

The loss function in \autoref{eq:loss-theory} is convex in $\weights$ and therefore admits a unique minimizer $\weights_\nObs\in\R^\paramdim$, which we can derive in closed form as:
\begin{equation}
    \begin{split}
        \nabla\loss_\nObs(\weights) &= \features(\features^\transpose\weights - \observations) + \regFactor(\weights - \weights_0)\\
        \nabla\loss_\nObs(\weights)\big|_{\weights=\weights_\nObs} = \vec 0 &\implies (\features\features^\transpose + \regFactor \eye)\weights_\nObs = \features\observations + \regFactor\weights_0\,.
    \end{split}
\end{equation}
For $\regFactor > 0$, the matrix on the left-hand side is positive-definite, and therefore invertible, then:
\begin{equation}
    \weights_\nObs = (\features\features^\transpose + \regFactor\eye)^{-1} (\features\observations + \regFactor\weights_0)\,.
\end{equation}
Suppose $\weights_0 \sim \normal(\vec 0, \eye)$. Then $\weights_\nObs|\observations \sim \normal(\vecmean\weights_\nObs, \vecmean\paramcov_\nObs)$, where:
\begin{equation}
    \vecmean{\weights}_\nObs := \expectation[\weights_\nObs \mid \observations] = (\features\features^\transpose + \regFactor\eye)^{-1} \features\observations\,,
\end{equation}
and the covariance matrix is given by:
\begin{equation}
    \begin{split}
        \vecmean\paramcov_\nObs := \variance[\weights_\nObs \mid \observations] &=
        \variance[(\features\features^\transpose + \regFactor\eye)^{-1} (\features\observations + \regFactor\weights_0) \mid \observations]
        \\
        &=
        \variance[\regFactor(\features\features^\transpose + \regFactor\eye)^{-1} \weights_0]
        \\
        &=\regFactor^2(\features\features^\transpose + \regFactor\eye)^{-1}\variance[\weights_0](\features\features^\transpose + \regFactor\eye)^{-1}
        \\
        &= 
        \regFactor^2(\features\features^\transpose + \regFactor\eye)^{-2}\,,
    \end{split}
\end{equation}
where we used the fact that $\variance[\mat A \weights] = \mat A \variance[\weights] \mat A^\transpose$ for a random vector $\weights$, and we also note that $\variance[\weights_0 \mid \observations] = \variance[\weights_0]$, given that $\weights_0$ is sampled independently of $\observations$.

\paragraph{Alternative derivation.} Another way of deriving the expression above is via the joint distribution between $\weights_\nObs$ and $\observations$.
Assume $\observations = \features^\transpose\weights_* + \vec\obsnoise$, for some $\weights_* \sim \normal(\vec 0, \eye)$ and $\vec\obsnoise \sim \normal(\vec 0, \sigma_\obsnoise^2\eye)$, so that $\paramcov_\observations := \variance[\observations] = \features\features^\transpose + \sigma_\obsnoise^2\eye$. The joint distribution is:
\begin{equation}
    \begin{bmatrix}
        \weights_\nObs\\
        \observations
    \end{bmatrix}
    \sim
    \normal
    \left(
        \begin{bmatrix}
            \vec 0\\
            \vec 0
        \end{bmatrix}
        ,
        \begin{bmatrix}
            (\features\features^\transpose + \regFactor\eye)^{-1} (\features\paramcov_\observations\features^\transpose + \regFactor^2\eye) (\features\features^\transpose + \regFactor\eye)^{-1}
            &
            (\features\features^\transpose + \regFactor\eye)^{-1}\features\paramcov_\observations
            \\
            \paramcov_\observations\features^\transpose(\features\features^\transpose + \regFactor\eye)^{-1}
            &
            \paramcov_\observations
        \end{bmatrix}
    \right).
\end{equation}
The covariance of the joint distribution is obtained from the linear relation between $\weights_\nObs$ and $\observations$ as:
\begin{gather*}
\mat\Sigma_{\weights_\nObs, \observations} =
\begin{bmatrix}
(\features\features^\transpose + \regFactor\eye)^{-1}  & \vec 0\\
\vec 0 & \eye
\end{bmatrix}
\left(
\begin{bmatrix}
\features \\
\eye
\end{bmatrix}
\paramcov_\observations
\begin{bmatrix}
\features \\
\eye
\end{bmatrix}^\transpose
+
\begin{bmatrix}
\regFactor^2\eye & \vec 0\\
\vec 0 & \vec 0
\end{bmatrix}
\right)
\begin{bmatrix}
(\features\features^\transpose + \regFactor\eye)^{-1}  & \vec 0\\
\vec 0 & \eye
\end{bmatrix}.
\end{gather*}
We can see that the matrix above is non-singular and positive definite. In particular, its determinant can be derived as:
\begin{equation*}
    \begin{split}
        \det(\mat\Sigma_{\weights_\nObs, \observations})
        &=
        \det
            \begin{pmatrix}
            (\features\features^\transpose + \regFactor\eye)^{-1}  & \vec 0\\
            \vec 0 & \eye
            \end{pmatrix}
            ^2
        \det
        \left(
            \begin{bmatrix}
            \features \\
            \eye
            \end{bmatrix}
            \paramcov_\observations
            \begin{bmatrix}
            \features \\
            \eye
            \end{bmatrix}^\transpose
            +
            \begin{bmatrix}
            \regFactor^2\eye & \vec 0\\
            \vec 0 & \vec 0
            \end{bmatrix}
        \right)\\
        &=
        \det(\features\features^\transpose + \regFactor\eye)^{-2}
        \det
        \left(
            \begin{bmatrix}
            \features\paramcov_\observations\features^\transpose
            &
            \features \paramcov_\observations\\
            \paramcov_\observations\features^\transpose
            &
            \paramcov_\observations
            \end{bmatrix}
            +
            \begin{bmatrix}
            \regFactor^2\eye & \vec 0\\
            \vec 0 & \vec 0
            \end{bmatrix}
        \right)\\
        &=
        \det(\features\features^\transpose + \regFactor\eye)^{-2}
        \det
        \left(
            \begin{bmatrix}
            \features\paramcov_\observations\features^\transpose + \regFactor^2\eye
            &
            \features \paramcov_\observations\\
            \paramcov_\observations\features^\transpose
            &
            \paramcov_\observations
            \end{bmatrix}
        \right)\\
        &=
        \det(\features\features^\transpose + \regFactor\eye)^{-2}
        \det(\paramcov_\observations)
        \det(
            \features\paramcov_\observations\features^\transpose + \regFactor^2\eye
            -
            \features \paramcov_\observations \paramcov_\observations^{-1} \paramcov_\observations\features^\transpose
        )\\
        &=
        \frac{
            \det(\paramcov_\observations)
            \det(\regFactor^2\eye)
        }
        {
            \det(\features\features^\transpose + \regFactor\eye)^{2}
        }\\
        & > 0\,,
    \end{split}
\end{equation*}
where the inequality holds as long as $\regFactor > 0$ and $\sigma_\obsnoise > 0$.
Conditioning on $\observations$ then yields:
\begin{align}
    \vecmean\weights_\nObs &= (\features\features^\transpose + \regFactor\eye)^{-1} \features\observations\,,
\end{align}
and:
\begin{equation}
    \begin{split}
        \vecmean\paramcov_\nObs &= (\features\features^\transpose + \regFactor\eye)^{-1} (\features\paramcov_\observations\features^\transpose + \regFactor^2\eye) (\features\features^\transpose + \regFactor\eye)^{-1} - (\features\features^\transpose + \regFactor\eye)^{-1}\features\paramcov_\observations\features^\transpose(\features\features^\transpose + \regFactor\eye)^{-1}\\
        &=\regFactor^2(\features\features^\transpose + \regFactor\eye)^{-2}\,.
    \end{split}
\end{equation}
In contrast, even if $\regFactor := \sigma_\obsnoise^2$, note that $\vecmean\paramcov_\nObs$ does not correspond to the exact posterior covariance, which can be derived as:
\begin{align}
    \begin{bmatrix}
        \weights_*\\
        \observations
    \end{bmatrix}
    &\sim
    \normal
    \left(
        \begin{bmatrix}
            \vec 0\\
            \vec 0
        \end{bmatrix}
        ,
        \begin{bmatrix}
            \eye
            &
            \features
            \\
            \features^\transpose
            &
            \features^\transpose\features + \regFactor\eye
        \end{bmatrix}
    \right)\,.
    \\
    \implies \paramcov_\nObs := \variance[\weights_* \mid \observations] &= \eye - \features(\features^\transpose\features + \regFactor\eye)^{-1}\features^\transpose = \regFactor(\features\features^\transpose + \regFactor\eye)^{-1}\,.
    \label{eq:exact-weights-posterior-covariance}
\end{align}

\paragraph{Predictions.} For the predictive equations, note that adding and subtracting $\features\features^\transpose\weights_0$ to the expression for $\weights_\nObs$ yields:
\begin{equation}
    \begin{split}
        \weights_\nObs &= (\features\features^\transpose + \regFactor\eye)^{-1} (\features\observations + \regFactor\weights_0 + \features\features^\transpose\weights_0 - \features\features^\transpose\weights_0)\\
        &= \weights_0 + (\features\features^\transpose + \regFactor\eye)^{-1} (\features\observations - \features\features^\transpose\weights_0)\\
        &= \weights_0 + \features(\features^\transpose\features + \regFactor\eye)^{-1}(\observations - \features^\transpose\weights_0)\,,
    \end{split}
\end{equation}
where we applied the identity $(\eye + \mat{AB})^{-1}\mat A = \mat A (\eye + \mat{BA})^{-1}$.
Hence, letting $\model_\nObs(\location) := \ffeature(\location)^\transpose\weights_\nObs$, we have that:
\begin{equation}
    \model_\nObs(\location) = \model_0(\location) + \ffeature(\location)^\transpose\features(\features^\transpose\features + \regFactor\eye)^{-1}(\observations - \vec\model_0)\,,
\end{equation}
where $\vec\model_0 := \features^\transpose\weights_0 = [\model_0(\location_i)]_{i=1}^\nObs \in \R^\nObs$. In the infinite-width limit, we then have that:
\begin{equation}
    \model_\nObs(\location) = \model_0(\location) + \vec\kernel_\nObs(\location)^\transpose(\Kernel_\nObs + \regFactor\eye)^{-1}(\observations - \vec\model_0)\,,
\end{equation}
where we set $\kernel := \kernel_\nngp$ and adopt the standard GP notation for the kernel vector $\vec\kernel_\nObs$ and matrix $\Kernel_\nObs$.

\paragraph{Underestimated variance.} Now considering $\model_0 \sim \gp(0, \kernel)$, we have that:
\begin{align}
    \expectation[\model_\nObs(\location) \mid \observations] &= \vec\kernel_\nObs(\location)^\transpose(\Kernel_\nObs + \regFactor\eye)^{-1}\observations\\
    \begin{split}
        \variance[\model_\nObs(\location) \mid \observations] &= \kernel(\location,\location) - 2\vec\kernel_\nObs(\location)^\transpose(\Kernel_\nObs + \regFactor\eye)^{-1}\vec\kernel_\nObs(\location)\\
        &\quad + \vec\kernel_\nObs(\location)^\transpose(\Kernel_\nObs + \regFactor\eye)^{-1}\Kernel_\nObs(\Kernel_\nObs + \regFactor\eye)^{-1}\vec\kernel_\nObs(\location)\\
        &= \kernel(\location,\location) - \vec\kernel_\nObs(\location)^\transpose(\Kernel_\nObs + \regFactor\eye)^{-1}\vec\kernel_\nObs(\location) - \regFactor\vec\kernel_\nObs(\location)^\transpose(\Kernel_\nObs + \regFactor\eye)^{-2}\vec\kernel_\nObs(\location)\,,
    \end{split}
    \label{eq:underestimated-posterior-variance}
\end{align}
where the last equality follows by adding and subtracting $\regFactor\eye$ from the $\Kernel_\nObs$ factor in the previous quadratic term.
We can then see that the predictive variance is lower than the exact GP posterior predictive variance by a factor of $\regFactor\vec\kernel_\nObs(\location)^\transpose(\Kernel_\nObs + \regFactor\eye)^{-2}\vec\kernel_\nObs(\location)$. The two match when $\regFactor \to 0$, as in \citet{Lee2019}. However, for the noisy case with $\regFactor >0$, we have this mismatch, as it can also be observed in the results of \citet{Ordonez2025observation}.  
Similarly, for the weights posterior covariance, we have that:
\begin{equation}
    \begin{split}
        \vecmean\paramcov_\nObs
        =\regFactor^2(\features\features^\transpose + \regFactor\eye)^{-2}
        &\preceq \regFactor(\features\features^\transpose + \regFactor\eye)^{-1} = \paramcov_\nObs\\
        \iff \regFactor(\features\features^\transpose + \regFactor\eye)^{-2}
        &\preceq (\features\features^\transpose + \regFactor\eye)^{-1}\\
        \iff \regFactor(\features\features^\transpose + \regFactor\eye)^{-1}
        &\preceq \eye\\
        \iff (\regFactor^{-1}\features\features^\transpose + \eye)^{-1}
        &\preceq \eye\,,
    \end{split}
    \label{eq:underestimated-covariance}
\end{equation}
which holds since $\features\features^\transpose$ is positive semidefinite and $\regFactor > 0$. Hence, in the following we analyze the effect of the underestimated variance on the algorithm's regret.

\paragraph{Effect on the regret bound.} We may bound the effect of the posterior variance mismatch in the regret bound of GP-TS. Let $\paramcov_\iterIdx = \variance[\weights_*|\observations]$ represent the exact posterior covariance matrix (cf. \autoref{eq:exact-weights-posterior-covariance}) after $\iterIdx \geq 1$ iterations, assuming $\regFactor := \sigma_\obsnoise^2$, and denote the exact and the approximate posterior, respectively, as:
\begin{align}
    \pmeasure_\iterIdx &:= \normal(\vecmean\weights_\iterIdx, \paramcov_\iterIdx)\\
    \hat\pmeasure_\iterIdx &:= \normal(\vecmean\weights_\iterIdx, \vecmean\paramcov_\iterIdx)\,.
\end{align}
Correspondingly, we set:
\begin{align}
    \location^* &\in \argmax_{\location\in\domain} \objective(\location)\\
    \location_\iterIdx &\in \argmax_{\location\in\domain} \model_\iterIdx(\location)\,,
\end{align}
assuming $\objective(\location) = \ffeature(\location)^\transpose\weights_*$, for some $\weights_* \sim \normal(\vec 0, \eye)$. The instant regret at iteration $\iterIdx\geq 1$ is then:
\begin{equation}
    \begin{split}
        \expectation[\objective(\location^*) - \objective(\location_\iterIdx)] 
        &= \expectation[\expectation[\objective(\location^*) - \objective(\location_\iterIdx) \mid \dataset_{\iterIdx-1}]]\\
        &= \expectation\left[
            \int_{\R^\paramdim} \int_{\R^\paramdim} \objective(\location^*) - \objective(\location_\iterIdx) \diff\pmeasure_{\iterIdx-1}(\weights_*) \diff\hat\pmeasure_{\iterIdx-1}(\weights_\iterIdx)
        \right]\\
        &= \expectation\left[
            \int_{\R^\paramdim} \int_{\R^\paramdim} 
            (\objective(\location^*) - \objective(\location_\iterIdx))
            \frac{\diff\hat\pmeasure_{\iterIdx-1}}{\diff\pmeasure_{\iterIdx-1}}(\weights_\iterIdx)
            \diff\pmeasure_{\iterIdx-1}(\weights_*) \diff\pmeasure_{\iterIdx-1}(\weights_\iterIdx)
        \right]\\
        &\leq \expectation\left[
            \left\lVert 
                \frac{\diff\hat\pmeasure_{\iterIdx-1}}{\diff\pmeasure_{\iterIdx-1}} 
            \right\lVert_\infty
            \int_{\R^\paramdim} \int_{\R^\paramdim} 
            \objective(\location^*) - \objective(\location_\iterIdx) \diff\pmeasure_{\iterIdx-1}(\weights_*)
            \diff\pmeasure_{\iterIdx-1}(\weights_\iterIdx)
        \right]\,,
    \end{split}
\end{equation}
where we applied H\"{o}lder's inequality, noting that $\objective(\location^*) - \objective(\location_\iterIdx) \geq 0$. Therefore, if the Radon-Nikodym derivative $\frac{\diff\hat\pmeasure_{\iterIdx-1}}{\diff\pmeasure_{\iterIdx-1}}$ is uniformly bounded, the regret bound remains the same. In the finite-width case $\paramdim < \infty$, the density ratio between multivariate normal distributions with the same mean gives us:
\begin{equation}
    \frac{\diff\hat\pmeasure_{\iterIdx}}{\diff\pmeasure_{\iterIdx}}(\weights)
    =
    \sqrt{\frac{\det(\paramcov_\iterIdx)}{\det(\vecmean\paramcov_\iterIdx)}}
    \exp\left(
        - \frac{1}{2}
        (\weights - \vecmean\weights_\iterIdx)^\transpose
        (\vecmean\paramcov_\iterIdx^{-1} - \paramcov_\iterIdx^{-1})
        (\weights - \vecmean\weights_\iterIdx)
    \right)
    \,, \quad \weights\in\R^\paramdim\,.
\end{equation}
As $\vecmean\paramcov_\iterIdx \preceq \paramcov_\iterIdx$ \eqref{eq:underestimated-covariance}, the difference between the inverses $\vecmean\paramcov_\iterIdx^{-1} - \paramcov_\iterIdx^{-1}$ is positive semidefinite. The maximum is then achieved at $\weights = \vecmean\weights_\iterIdx$, yielding:
\begin{equation}
    \begin{split}
        \flexnorm{
            \frac{\diff\hat\pmeasure_{\iterIdx}}{\diff\pmeasure_{\iterIdx}}
        }_\infty
        &= \sqrt{\frac{\det(\paramcov_\iterIdx)}{\det(\vecmean\paramcov_\iterIdx)}}\\
        &= \sqrt{
            \det\left(\paramcov_\iterIdx \vecmean\paramcov_\iterIdx^{-1}\right)
        }\\
        &= \sqrt{
            \det\left(
                \eye + \regFactor^{-1}\features\features^\transpose
            \right)
        }\\
        &= \sqrt{
            \det\left(
                \eye + \regFactor^{-1}\features^\transpose\features
            \right)
        }
    \end{split}
\end{equation}
where we applied Sylvester's determinant identity to third line, and a standard determinant identity yields the last equality. In the infinite-width limit as $\paramdim\to\infty$, we have that $\features^\transpose\features$ converges to $\Kernel_\iterIdx := [\kernel_\nngp(\location_i, \location_j)]_{i,j=1}^\iterIdx$, leading us to:
\begin{equation}
        \flexnorm{
            \frac{\diff\hat\pmeasure_{\iterIdx}}{\diff\pmeasure_{\iterIdx}}
        }_\infty
    = \sqrt{\det(\eye + \regFactor^{-1} \Kernel_\iterIdx)}\,.
\end{equation}
Recall the definition of the maximum information gain \citep{Srinivas2010, Takeno2024}:
\begin{equation}
    \mig_\iterIdx := \max_{\domain_\iterIdx \subset \domain : \card{\domain_\iterIdx} \leq \iterIdx} \frac{1}{2} \log\det(\eye + \regFactor^{-1} \Kernel_\iterIdx)\,.
    \label{eq:mig}
\end{equation}
If we assume that the GP information gain $\frac{1}{2}\log\det(\eye + \regFactor^{-1} \Kernel_\iterIdx)$ is bounded by $\mig_\iterIdx$ as above, we would then have that:
\begin{equation}
    \flexnorm{
            \frac{\diff\hat\pmeasure_{\iterIdx}}{\diff\pmeasure_{\iterIdx}}
    }_\infty
    \leq
    \exp \mig_\iterIdx\,,
\end{equation}
which is usually an unbounded term, given that $\mig_\iterIdx$ is a non-decreasing function of $\iterIdx$. However, for a finite domain $\card{\domain} < \infty$, we trivially have that $\mig_\iterIdx \leq \mig_{\card{\domain}}$, given that the largest finite subset $\domain_\iterIdx$ of $\domain$ is $\domain$ itself. Hence, in this case, the following holds:
\begin{equation}
    \forall \iterIdx\in\N\,,
    \quad
    \flexnorm{
            \frac{\diff\hat\pmeasure_{\iterIdx}}{\diff\pmeasure_{\iterIdx}}
    }_\infty
    \leq
    \exp \mig_{\card{\domain}}\,,
\end{equation}
which is bounded for most practical kernels. Putting it all together, we have that:
\begin{equation}
     \forall \iterIdx\in\N\,, \quad \expectation[\objective(\location^*) - \objective(\location_\iterIdx)] \leq \regret_\iterIdx\exp \mig_{\card{\domain}}\,,
     \label{eq:regret-mig-bound}
\end{equation}
where $\regret_\iterIdx$ represents the Bayesian regret when $\location_\iterIdx$ maximizes a sample from the exact GP posterior, instead of its approximation. Given that $\mig_{\card{\domain}}$ is a finite constant, the asymptotic rates for the Bayesian cumulative regret remain the same even in the presence of an underestimated predictive variance.

\paragraph{Problem with $\mig_\iterIdx$ bound.} An issue with the finite bound on the Radon-Nikodym derivative above can be found when contrasting the classic definition of the maximum information gain $\mig_\iterIdx$ in the literature \eqref{eq:mig} with the actual information gain in the algorithm, i.e., the mutual information between an exact GP and the collected observations, which can be shown to be quantified by $\frac{1}{2}\log\det(\eye + \regFactor^{-1} \Kernel_\iterIdx)$ \citep{Srinivas2010}. The issue is that, although $\mig_\iterIdx \leq \mig_{\card{\domain}}$ from the definition commonly found in the literature \citep{Srinivas2010, Chowdhury2017}, it does not necessarily follow that the actual information gain is bounded after we account for multiplicities in the eigenvalues. The algorithm is in principled allowed (and likely) to make repeated choices of the same $\location_\iterIdx = \location_*$ for all $\iterIdx$ at some point, for a fixed $\location_* \in \domain$, which may or may not be the optimizer $\location^*$. In the simplest case, if all of the algorithm's choices are made at any fixed $\location_*\in\domain$, we have that:
\begin{equation}
    \begin{split}
        \log \det(\eye + \regFactor^{-1} \Kernel_\iterIdx)
        &= \log\det(\eye + \regFactor^{-1} \features_\iterIdx^\transpose\features_\iterIdx)\\
        &= \log\det(\eye + \regFactor^{-1} \features_\iterIdx\features_\iterIdx^\transpose)\\
        &= \log\det(\eye + \iterIdx\regFactor^{-1}\ffeature(\location^*)\otimes\ffeature(\location^*))\\
        &= \log\det(1 + \iterIdx\regFactor^{-1}\norm{\ffeature(\location^*)}^2)\\
        &\geq \log (1 + \anyconstant \iterIdx),
    \end{split}
\end{equation}
for some constant $\anyconstant > 0$. Therefore, this lower bound diverges as $\iterIdx\to\infty$, whereas $\mig_\iterIdx \leq \mig_{\card{\domain}} < \infty$ remains bounded, leading to a contradiction of the previous conclusion in \autoref{eq:regret-mig-bound}.

\paragraph{Exactly matching the posterior variance.} The exact weights posterior covariance $\paramcov_\iterIdx$ can be matched if, besides randomizing the initial weights, we randomize the observations by adding noise $\vec{\tilde{\obsnoise}}_\iterIdx \sim \normal(\vec 0, \regFactor\eye)$ to them at training time, following a \emph{randomize-then-optimize} approach \citep{Bardsley2014}. Specifically, we minimize the perturbed loss:
\begin{equation}
    \tilde\loss_\iterIdx(\weights) := \frac{1}{2}\norm{\features_\iterIdx^\transpose\weights - (\observations_\iterIdx + \vec{\tilde{\obsnoise}}_\iterIdx)}^2 + \frac{\regFactor}{2} \norm{\weights - \weights_0}^2\,, \quad \vec{\tilde{\obsnoise}}_\iterIdx \sim \normal(\vec 0, \regFactor\eye), \quad \weights_0 \sim \normal(\vec 0, \eye).
    \label{eq:loss-rto}
\end{equation}
The corresponding minimizer is given by:
\begin{equation}
    \vec{\tilde\weights}_\iterIdx = (\features_\iterIdx\features_\iterIdx^\transpose + \regFactor\eye)^{-1} (\features_\iterIdx(\observations_\iterIdx + \vec{\tilde\obsnoise}_\iterIdx) + \regFactor\weights_0),
\end{equation}
whose conditional mean still matches the exact weights posterior mean:
\begin{align}
    \expectation[\vec{\tilde\weights}_\iterIdx \mid \observations_\iterIdx] &= (\features_\iterIdx\features_\iterIdx^\transpose + \regFactor\eye)^{-1} \features_\iterIdx\observations_\iterIdx
\end{align}
and whose conditional covariance now satisfies:
\begin{equation}
    \begin{split}
        \variance[\vec{\tilde\weights}_\iterIdx \mid \observations_\iterIdx] 
        &= \variance[ (\features_\iterIdx\features_\iterIdx^\transpose + \regFactor\eye)^{-1} (\features_\iterIdx(\observations_\iterIdx + \vec{\tilde\obsnoise}_\iterIdx) + \regFactor\weights_0) \mid \observations_\iterIdx ]\\
        &= (\features_\iterIdx\features_\iterIdx^\transpose + \regFactor\eye)^{-1} 
        \variance[ \features_\iterIdx\vec{\tilde\obsnoise}_\iterIdx + \regFactor\weights_0 ]
        (\features_\iterIdx\features_\iterIdx^\transpose + \regFactor\eye)^{-1}\\
        &= (\features_\iterIdx\features_\iterIdx^\transpose + \regFactor\eye)^{-1} (\regFactor\features_\iterIdx\features_\iterIdx^\transpose + \regFactor^2\eye) (\features_\iterIdx\features_\iterIdx^\transpose + \regFactor\eye)^{-1}\\
        &= \regFactor(\features_\iterIdx\features_\iterIdx^\transpose + \regFactor\eye)^{-1}\\
        &= \paramcov_\iterIdx,
    \end{split}
\end{equation}
thereby recovering the exact posterior covariance. Nevertheless, note that the original difference in posterior predictive variance according to \autoref{eq:underestimated-posterior-variance} is $\regFactor\vec\kernel_\iterIdx(\location)^\transpose(\Kernel_\iterIdx + \regFactor\eye)^{-2}\vec\kernel_\iterIdx(\location)$, which is typically negligible for small values of noise variance $\sigma_\obsnoise^2 = \regFactor$. As a consequence, our cumulative regret bound remains approximately valid for NOTS, despite the slight underestimation of the posterior variance.

\section{Experiment details}
\label{app:experiments}


\subsection{Darcy flow}
Darcy flow describes the flow of a fluid through a porous medium with the following PDE form
\begin{align*}
    -\nabla \cdot (a(x)\nabla u(x)) &= g(x), \quad  x\in \Omega = (0, 1)^2\\
    u(x) &= 0, \quad x \in \partial \Omega,
\end{align*}
where \(u(x)\) is the flow pressure,  \(a(x)\) is the permeability coefficient and  \(g(x)\) is the forcing function. We fix \(g(x)=1\) and generate different solutions at random with zero Neumann boundary conditions on the Laplacian, following the setting in \citet{Li2021fno}, as implemented by the neural operator package \citep{kossaifi2024neural}. In particular, for this problem, we generate a search space $\Sspace$ with $\card{\Sspace} = 1000$ data points. The divergence of \(f\) is \(\nabla \cdot f = \frac{\partial f_x}{\partial x} + \frac{\partial f_y}{\partial y}\) where \(f:  \Omega \rightarrow \mathbb{R}^2\) is a vector field \(f=(f_x, f_y)\). The gradient \(\nabla u = (\frac{\partial u(x, y)}{\partial x}, \frac{\partial u(x,y)}{\partial y})\) where \(u(x, y): \Omega \rightarrow \mathbb{R}\) is a scalar field.
Inspired by previous works \citep{wiker2007topology, katz1979history, mao2023physics}, we chose the following objective functions to evaluate the functions assuming that we aim to maximize the objective function \(f(\cdot)\):

\begin{enumerate}
    \item Negative total flow rates \citep{katz1979history}
\[f(u, a) = \int_{\partial\Omega} a(x) (\nabla u(x) \cdot n)ds\]
where \(s=\partial\Omega\) is the boundary of the  domain and \(n\) is the outward pointing unit normal vector of the boundary. \(q(x) = -a(x)\nabla u(x)\) is the volumetric flux which describes the rate of volume flow across a unit area. Therefore, the objective function measures the boundary outflux. Since the boundary is defined on a grid, \(n \in \{[-1, 0],  [1, 0], [0, 1], [0, -1]\}\) for the left, right, top and bottom boundaries. The boundary integral can be simplified as
\[\int_0^1 [-a(0,y) u_x(0, y) + a(1, y)u_x(1, y)]dy + \int_0^1 [-a(x, 0)u_y(x,0) +a(x,1)u_y(x,1)]dx \]
where $u_x(x,y)=\frac{\partial u}{\partial x}, u_y(x,y)=\frac{\partial u}{\partial y}$


\item  Negative total pressure (Eq 2.1 in \citep{jeong2025optimal})

\[f(u, g) = -\frac{1}{2} \int_\Omega (\| u(x)\|_2 + \beta \|g(x)\|_2 )dx\]

with \(\beta > 0\) is a coefficient for the forcing term \(g(x)\). With a constant \(g(x)\), the objective is simplified as \(  -\frac{1}{2} \int_\Omega \| u(x)\|_2 dx\).

\item Negative total potential energy \citep{wiker2007topology}
\[
\objective(\outfun,\infun) = -\int_\domain \infun(\location)\norm{\nabla\outfun(\location)}^2\diff\location + \int_\domain s(\location)\outfun(\location)\diff\location
\]
This functional corresponds to the system's total potential energy. It balances the energy dissipated by fluid friction (the first term) against the potential energy supplied by the uniform fluid source (the second term, where $s=1$ is assumed). In our design optimization context, where the underlying physical state $u$ is already a stable solution to the Darcy PDE, minimization of this functional over the set of permeability fields $\infun \in \Sspace$ determines the permeability field $\infun^*$ that requires the minimum total energy to sustain the required fluid injection (source $s=1$) while maintaining zero pressure at the boundary ($u=0$). This effectively identifies the most hydrodynamically efficient design for the given flow constraints. This functional is related to the potential power functional in \citet{wiker2007topology} with the difference that the latter requires estimates of the velocity field, while the simplified energy calculation above only uses the pressure field $\outfun$.
\end{enumerate}

\subsection{Shallow Water}
The shallow water equation on the rotating sphere is often used to model ocean waters over the surface of the globe. This problem can be described by the following PDE \citep{bonev2023spherical}:
\[
  \frac{\partial \varphi}{\partial t} + \nabla \cdot (\varphi v) = 0  \quad \text{in} \quad \mathbb{S}^2 \times \{0, +\infty\}
\]
\[
 \frac{\partial (\varphi v)}{\partial t} + \nabla \cdot (\varphi v\otimes v) = g  \quad \text{in} \quad \mathbb{S}^2 \times \{0, +\infty\}
\]
\[
\varphi=\varphi_0,  \quad v=v_0 \quad \text{on} \quad \mathbb{S}^2 \times \{0\}
\]
where the input function is defined as the initial condition of the state \(a = (\varphi_0, \varphi_0 v_0)\) with the geopotential layer depth \(\varphi\) and the discharge (\(v\) is the velocity field), \(g\) is the Coriolis force term, and $\mathbb{S}^2$ denotes the surface of the 2-sphere in $\R^3$. The output function \(u\) predicts the state function at time $t$: \((\varphi_t, \varphi_t v_t)\). For this problem, we use a search space $\Sspace$ with $\card{\Sspace} = 200$ data points.

As the shallow water equation is usually chosen as a simulator of global atmospheric variables, we adopt the most common data assimilation objective \citep{rabier1998extended, xiao2023fengwu} in the weather forecast literature defined as:
\[
    f(u, a) = \frac{1}{2}\langle a - a_p, B^{-1}(a - a_p)\rangle + \frac{1}{2} \langle u-u_{t}, R^{-1}(u- u_{t})\rangle,
\]
where \(a_p\) describes the prior estimate of the initial condition, \(u_t\) represents the ground truth function, the background kernel \(B\) and error kernel \(R\) can be computed with historical data. The objective can be defined as an inverse problem which corresponds to finding the initial condition \(a\) that generates the ground truth solution function \(u_t\). Here we simplify the objective by not penalizing the initial condition (dropping the prior term) and assuming independence and unit variance on the solution functions using an identity kernel \(R\)), the simplified objective function \(f(u) = \frac{1}{2} \langle u-u_{t}, u- u_{t}\rangle\) can be used to measure different initial conditions.

\subsection{Noise} To simulate real-world settings, noise was added to the observations by computing the empirical covariance matrix of the outputs $\observation$ in the dataset for the corresponding PDE and then adding Gaussian noise with variance set to 1\% of the coordinate-wise output variance. 

\subsection{Algorithm settings}
\label{app:settings}
NOTS was implemented using the \emph{Neural Operator} library \citep{kossaifi2024neural} and run on NVIDIA H100 GPUs on CSIRO's high-performance computing cluster. For each dataset, we selected the recommended settings for FNO models according to examples in the library. Parameters were randomly initialized using Kaiming (or He) initialization \citep{He2015nninit} for the network weights, sampling from a normal distribution with variance inversely proportional to the input dimensionality of each layer, while biases were initialized to zero. For all experiments, we trained the model for 10 epochs of mini-batch stochastic gradient descent with an initial learning rate of $10^{-3}$ and a cosine annealing scheduler. The regularization factor for the L2 penalty was set as $\regFactor := 10^{-4}$. This same setting for the regularization factor was also applied to our implementation of STO-NTS.

\section{Additional results with single-hidden-layer model}
\label{app:extra}
More closely to the setting in our theoretical results, we tested a single-hidden-layer FNO on the Darcy flow PDE. Only the last hidden layer of the model was trained via full-batch gradient descent. The FNO was configured without any lifting layer, having only a single Fourier kernel convolution and a residual connection, as in the original formulation. The number of hidden channels was set to 2048 to approximate the infinite-width limit.

\begin{figure}[t]
    \centering
    \includegraphics[width=0.475\textwidth]{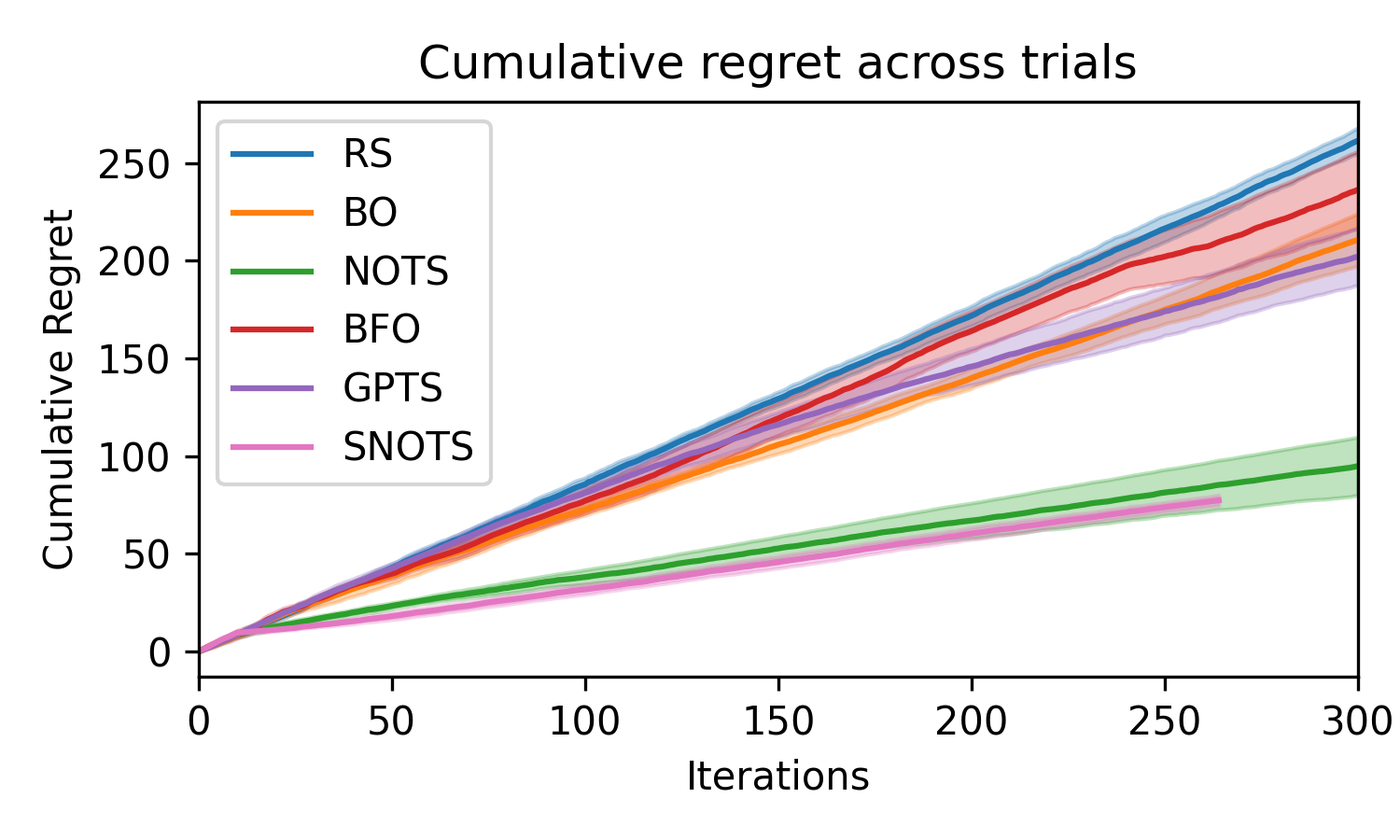}
    \includegraphics[width=0.475\textwidth]{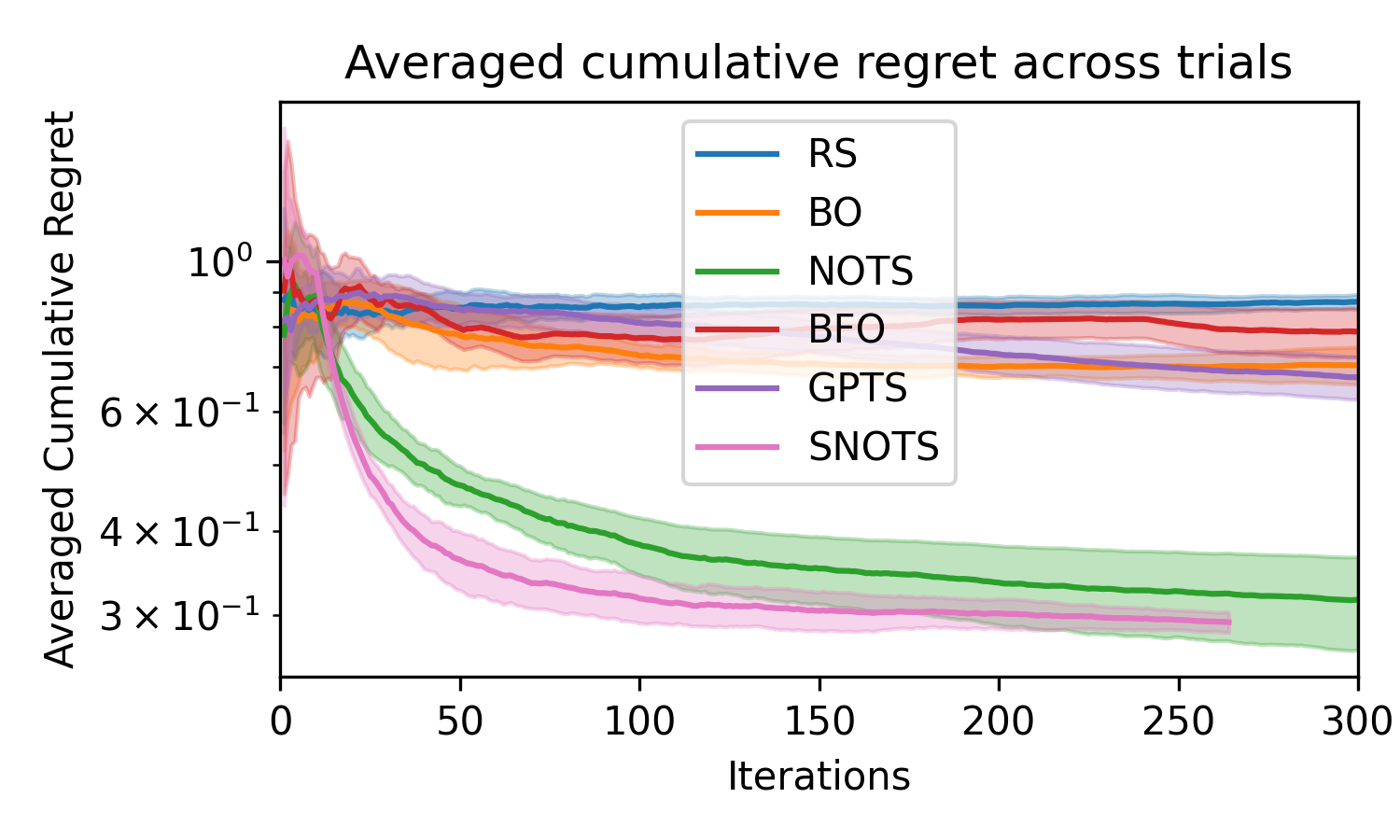}
    \caption{Cumulative regret across trials for the Darcy \emph{flow rate} optimization problem with only the last linear layer of a single-hidden-layer FNO trained via full-batch gradient descent for NOTS (labeled as SNOTS). All our results were averaged over 10 independent trials, and shaded areas represent $\pm 1$ standard deviation.}
    \label{fig:test-snots-flow}
\end{figure}

\begin{figure}[t]
    \centering
    \includegraphics[width=0.475\textwidth]{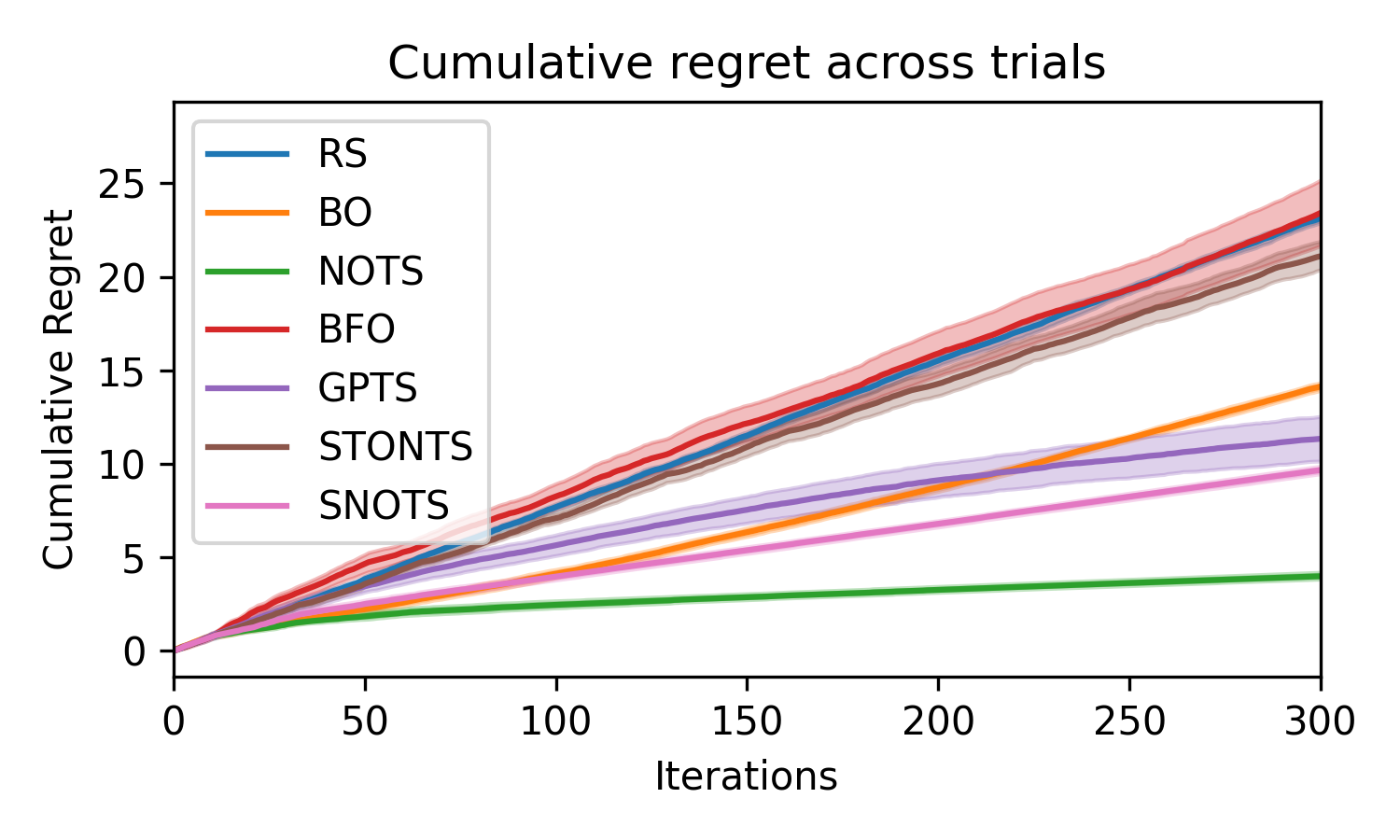}
    \includegraphics[width=0.475\textwidth]{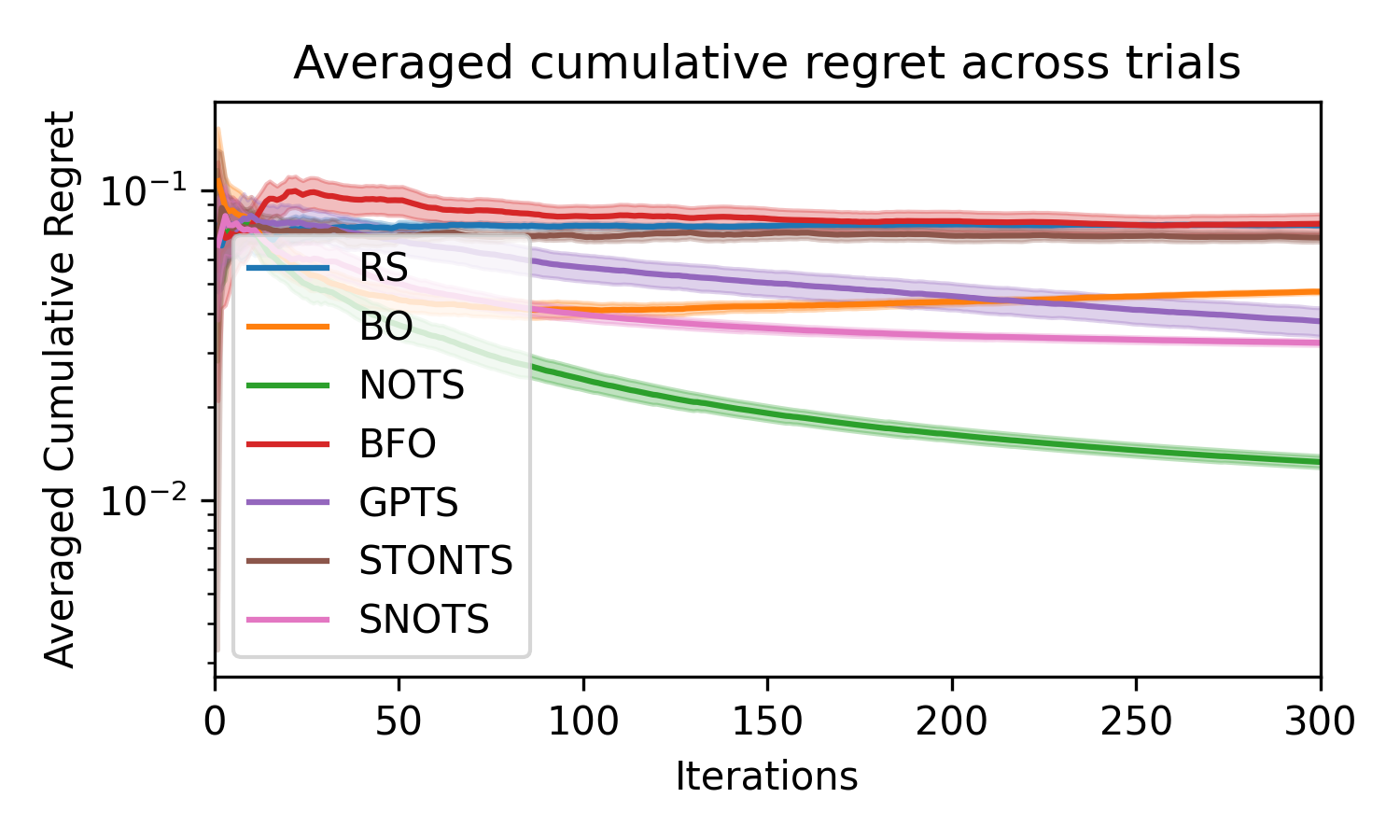}
    \caption{Cumulative regret across trials for the Darcy flow \emph{total pressure} optimization problem with only the last linear layer of a single-hidden-layer FNO trained via full-batch gradient descent for NOTS (labeled as SNOTS). 
    }
    \label{fig:test-snots-pressure}
\end{figure}

The results in \autoref{fig:test-snots-flow} show that the algorithm with the simpler model (SNOTS) can perform well in this setting, even surpassing the performance of the original NOTS. However, in the more challenging scenario imposed by the potential power problem \citep[adapted from][]{wiker2007topology}, we note that SNOTS struggles, only achieving mid-range performance when compared to other baselines, as shown in \autoref{fig:test-snots-pressure}. This performance drop suggests that the complexity of the pressure optimization problem may require more accurate predictions to capture details in the output functions that might heavily influence the potential power. In general, a quadratic objective will be more sensitive to small disturbances than a linear functional, hence requiring a more elaborate model.

\section{Limitations and extensions}
\label{app:limitations}

\paragraph{Noise.} We note that, although our result in Proposition 2 assumes a well specified noise model, it should be possible to show that the same holds for noise which is sub-Gaussian with respect to the regularization factor. The latter would allow for configuring the algorithm with any regularization factor which is at least as large as the assumed noise sub-Gaussian parameter (i.e., its variance if Gaussian distributed). However, this analysis can be quite involved and out of the immediate scope of this paper. Therefore, we leave such investigation for further research.

\paragraph{Nonlinear functionals.} We assumed a bounded linear functional in Proposition 2, which should cover a variety of objectives involving integrals and derivatives of the operator's output. However, this assumption may not hold for more interesting functionals, such as some objectives considered in our experiments. Similar to the case with noise, any Lipschitz continuous functional of the neural operator's output should follow a sub-Gaussian distribution \citep{Pisier1986}. Hence, the Gaussian approximation remains reasonable, though a more in-depth analysis would be needed to derive the exact rate of growth for the cumulative regret in these settings.

\paragraph{Mult-layer models.} For the theoretical analysis, we assumed a single hidden layer neural network as the basis of our Thompson sampling algorithm. While this choice provides a simple and computationally efficient framework, it may not be optimal for all applications or datasets. For instance, in some cases, a deeper neural network with more layers might provide better performance due to increased capacity to capture complex patterns in the data. Extending our analysis to this setting involves extending the inductive proofs for the multi-layer NNGP \citep{Lee2018,Matthews2018} to the case of neural operators. Such extension, however, may require transforming the operator layer's output back into a function in an infinite-dimensional space, which may lead to a bottleneck effect affecting the possibility of a kernel limit \citep{Liu2020}. In the single-hidden-layer case, such effect is avoided by operating directly with the finite-dimensional input function embedding $\nnkop(\infun)(\outlocation)\in \R^{\dimension_\nnkernel}$. Recently, concurrent work has explored the infinite-width limit for multi-layer neural operators \citep{deSouza2025}, but their applicability to NOTS is left as subject of future work.

\paragraph{Prior misspecification.} We assumed that the true operator $\nnop_*$ follows the same prior as our model, which was also considered to be infinitely wide. While this assumption greatly simplifies our analysis, more practical results may be derived by considering finite-width neural operators and a true operator which might not exactly correspond to a realization of the chosen class of neural operator models. For the case of finite widths, one simple way to obtain a similar regret bound is to let the width of the network grow at each Thompson sampling iteration. The approximation error between the GP model and the finite width neural operator can potentially be bounded as $\bigo(\paramdim^{-1/2})$ \citep{Liu2020}. Hence if the sequence of network widths $\{\paramdim_\iterIdx\}_{\iterIdx=1}^\infty$ is such that $\sum_{\iterIdx=1}^\infty \frac{1}{\sqrt{\paramdim_\iterIdx}} < \infty$, a similar regret bound to the one in Proposition 2 should be possible. Furthermore, if other forms of prior misspecification need to be considered, analyzing the Bayesian cumulative regret (instead of the more usual frequentist regret), as we did, allows one to bound the resulting cumulative regret of the misspecified algorithm via the Radon-Nikodym derivative $\frac{\diff\pmeasure}{\diff\hat\pmeasure}$ of the true prior $\pmeasure$ with respect to the algorithm's prior probability measure $\hat{\pmeasure}$. If its essential supremum $\left\lVert\frac{\diff\pmeasure}{\diff\hat\pmeasure}\right\rVert_\infty$ is bounded, then the resulting cumulative regret remains proportional to the same bound derived as if the algorithm's prior was the correct one \citep{Russo2014}.

\section{Broader impact}
\label{app:impact}
This work primarily focuses on the theoretical exploration of extending Thompson Sampling to function spaces via neural operators. As such, it does not directly engage with real-world applications or present immediate societal implications. However, the potential impact of this research lies in its application. By advancing methods for function-space optimization, this work may indirectly contribute to various fields that utilize complex simulations and models, such as climate science, engineering, and physics. Improvements in computational efficiency and predictive power in these fields could lead to positive societal outcomes, such as better climate modeling or engineering solutions. Nevertheless, any algorithm with powerful optimization capabilities carries ethical considerations. Its deployment in domains with safety-critical implications must be approached with care to avoid misuse or unintended consequences. Researchers and practitioners should ensure transparency, fairness, and accountability in applications potentially affecting society.

\end{document}